\documentclass{article}

\usepackage{graphicx,amsmath,amsfonts,amssymb,bm,hyperref,url,breakurl,epsfig,epsf,color,fullpage,MnSymbol,mathbbol,fmtcount,semtrans,cite,caption,subcaption,multirow,appendix}
\usepackage[]{algorithm2e}
\usepackage[noend]{algpseudocode}

\usepackage[utf8]{inputenc} % allow utf-8 input
\usepackage[T1]{fontenc}    % use 8-bit T1 fonts
\usepackage{hyperref}       % hyperlinks
\usepackage{url}            % simple URL typesetting
\usepackage{booktabs}       % professional-quality tables
\usepackage{amsfonts}       % blackboard math symbols
\usepackage{nicefrac}       % compact symbols for 1/2, etc.
\usepackage{microtype}      % microtypography

  \makeatletter
\def\BState{\State\hskip-\ALG@thistlm}
\makeatother

  \usepackage{mathtools}

\usepackage{titlesec}

\usepackage{tikz}
\usepackage{pgfplots}
\usetikzlibrary{pgfplots.groupplots}

\titleformat{\paragraph}
{\normalfont\normalsize\bfseries}{\theparagraph}{1em}{}
\titlespacing*{\paragraph}
{0pt}{3.25ex plus 1ex minus .2ex}{1.5ex plus .2ex}

\usepackage{movie15}

\usepackage{cite}
\usepackage{caption}
\usepackage[bottom,hang,flushmargin]{footmisc} 

\usepackage{hyperref}
\definecolor{darkred}{RGB}{150,0,0}
\definecolor{darkgreen}{RGB}{0,150,0}
\definecolor{darkblue}{RGB}{0,0,200}
\hypersetup{colorlinks=true, linkcolor=darkred, citecolor=darkgreen, urlcolor=darkblue}

\newtheorem{theorem}{Theorem}[section]
\newtheorem{assumption}{Assumption}[section]
\newtheorem{lemma}[theorem]{Lemma}

\newtheorem{proposition}[theorem]{Proposition}
\newtheorem{definition}[theorem]{Definition}

\newcommand{\eps}{\varepsilon}
\newcommand{\epsr}{\varepsilon_R}
\newcommand{\mur}{\mu_R}

\newcommand{\algo}{Regularized Binned Regression}
\newcommand{\bdata}{binned data matrix}
\newcommand{\irb}{{Binomial Binning}}
\newcommand{\rrb}{{Regular Binning}}

\newcommand{\beq}{\begin{equation}}
\newcommand{\bbetas}{{\bbeta_\star}}
\newcommand{\eeq}{\end{equation}}

\newcommand{\nn}{\nonumber}

\newcommand{\A}{{\mtx{A}}}

\newcommand{\B}{{{\mtx{B}}}}
\newcommand{\Xbin}{\mtx{\X}_{\rm{bin}}}
\newcommand{\Bbin}{\mtx{\B}}
\newcommand{\Bbinn}[1]{\mtx{\B}_{#1}}
\newcommand{\Xbinn}[1]{\mtx{\X}_{bin;#1}}
\newcommand{\Gbin}{{{\mtx{S}}}}
\newcommand{\xbin}{{{\mtx{b}}}}
\newcommand{\Sb}{{{\mtx{S}}}}

\newcommand{\pb}{p_{\rm{bin}}}
\newcommand{\Lc}{{\cal{L}}}

\newcommand{\Gc}{{\cal{G}}}

\newcommand{\Db}{{\mtx{D}}}
\newcommand{\oneb}{{\mathbb{1}}}
\newcommand{\onebb}{{\mathbf{1}}}
\newcommand{\Iden}{{\bf{I}}}
\newcommand{\M}{{\mtx{M}}}

\newcommand{\order}[1]{{\cal{O}}(#1)}
\newcommand{\mean}[1]{{\text{mean}}(#1)}
\newcommand{\z}{{\mtx{z}}}
\newcommand{\tn}[1]{\|{#1}\|_{\ell_2}}
\newcommand{\dist}[1]{\text{dist}(#1)}
\newcommand{\cone}[1]{{\text{cone}(#1)}}

\newcommand{\Cc}{\mathcal{C}}
\newcommand{\btrue}{\bbeta_{true}}
\newcommand{\bbeta}{{\boldsymbol{\beta}}}

\newcommand{\Sc}{\mathcal{S}}
\newcommand{\Ic}{\mathcal{I}}

\newcommand{\Nn}{\mathcal{N}}

\newcommand{\vb}{\mtx{v}}
\newcommand{\w}{\mtx{w}}

\newcommand{\li}{\left<}
\newcommand{\ri}{\right>}
\newcommand{\s}{\vct{s}}
\newcommand{\ab}{\vct{a}}

\newcommand{\ub}{\vct{u}}

\newcommand{\g}{\vct{g}}

\newcommand{\Tc}{\mathcal{T}}

\newcommand{\nnz}[1]{\texttt{nnz}(#1)}

\newcommand{\x}{\vct{x}}

\newcommand{\y}{\vct{y}}

\definecolor{emmanuel}{RGB}{255,127,0}

\newcommand{\p}{{\bf{p}}}
\newcommand{\R}{\mathbb{R}}
\newcommand{\Pro}{\mathbb{P}}

\newcommand{\E}{\operatorname{\mathbb{E}}}

\newcommand{\grad}[1]{{\nabla\Lc(#1)}}
\newcommand{\e}{\mathrm{e}}

\newcommand{\vct}[1]{\bm{#1}}
\newcommand{\mtx}[1]{\bm{#1}}

\newcommand{\supp}[1]{\operatorname{supp}(#1)}

\newcommand{\Pc}{{\cal{P}}}
\newcommand{\X}{{\bf{X}}}

\numberwithin{equation}{section} 

\def \endprf{\hfill {\vrule height6pt width6pt depth0pt}\medskip}
\newenvironment{proof}{\noindent {\bf Proof} }{\endprf\par}

\title{Learning Feature Nonlinearities with\\Non-Convex Regularized Binned Regression} 
\author{Samet Oymak$^\dagger$~~~Mehrdad Mahdavi$^\dagger$~~~Jiasi Chen$^\ddagger$\vspace{7pt}
\\$\dagger$ The Voleon Group\\
$\ddagger$ University of California, Riverside}
\date{May 19, 2017}
\begin{document}
\maketitle

% and kernel methods are commonly used to overcome these nonlinearities and they often outperform
\begin{abstract}For various applications, the relations between the dependent and independent variables are highly nonlinear. Consequently, for large scale complex problems, neural networks and regression trees are commonly preferred over linear models such as Lasso. This work proposes learning the feature nonlinearities by binning feature values and finding the best fit in each quantile using non-convex regularized linear regression. The algorithm first captures the dependence between neighboring quantiles by enforcing smoothness via piecewise-constant/linear approximation and then selects a sparse subset of good features. We prove that the proposed algorithm is statistically and computationally efficient. In particular, it achieves linear rate of convergence while requiring near-minimal number of samples. Evaluations on synthetic and real datasets demonstrate that algorithm is competitive with current state-of-the-art and accurately learns feature nonlinearities. Finally, we explore an interesting connection between the binning stage of our algorithm and sparse Johnson-Lindenstrauss matrices.\end{abstract}
\section{Introduction}
% samet look here 

% The algorithm first tries to enforce smoothness between neighboring quantiles by applying a piecewise-constant or piecewise-linear approximation and then selects a sparse subset of good features. 
%We prove that proposed algorithm is statistically and computationally efficient and derive sample complexity bounds.
%With the abundance of data, we need algorithms that can extract useful

% abstract copy:
%For various applications, the relations between the dependent and independent variables are highly nonlinear. Consequently, for large scale complex problems, neural networks and regression trees are commonly preferred over linear models such as Lasso. This work proposes learning the feature nonlinearities by binning feature values and finding the best fit in each quantile using non-convex regularized linear regression. The algorithm first selects a sparse subset of good features and then tries to enforce smoothness between neighboring quantiles by applying a piecewise-constant or piecewise-linear approximation. We prove that the proposed algorithm is statistically and computationally efficient. In particular, it achieves linear rate of convergence while requiring near-minimal number of samples. Evaluations on synthetic and real datasets demonstrate that algorithm is competitive with current state-of-the-art and accurately learns feature nonlinearities. Finally, we explore an interesting connection between the binning stage of our algorithm and Sparse Johnson-Lindenstrauss Transform.

Recently, substantial progress has been made on the problem of high-dimensional sparse linear models~\cite{o2016statistical}. In particular, Lasso has been shown to be remarkably successful, and is statistically well-behaved and generates interpretable solutions. However, in the presence of non-linearity (i.e., the relation between the covariates and response is non-linear), boosted decision trees, deep learning models, and kernel methods are regarded as the most effective models that deliver substantial performance boost over linear models; however, their interpretability is limited. As a result, there is a significant gap between the statistical performance and the interpretability, and it is often desirable to have computationally efficient algorithms that learn interpretable models without sacrificing statistical guarantees. This raises a natural question that we aim to tackle: \textit{Is there any algorithm which has similar statistical performance to complex models, while still retaining much of the interpretability of Lasso?}

%additive models with least-squares loss where the relation
In this paper, we answer the above question affirmatively and propose a novel way of learning the feature non-linearities with provable statistical and computational guarantees. In particular, we focus on additive models where, in case of least-squares loss, the relation between the response vector $\y\in\mathbb{R}^n$ and the data matrix $\X \in \mathbb{R}^{n \times p}$ is given by \cite{hastie1990generalized,ravikumar2007spam}:
\begin{align}
\hat{y}_i=\sum_{1\leq j\leq p}f_j(\X_{i,j}).\label{additive model}
\end{align}
Here, for each $j$, $\{f_j\}_{1\leq j\leq p}$ is the uni-variable feature gain function that we wish to learn. The broad idea is based on binning the continuous feature values \cite{liu2002discretization} and learning the correct gain of each quantile (see Figure \ref{fig:binning}). This strategy borrows ideas from high-dimensional estimation, feature discretization, and CART analysis \cite{dougherty1995supervised,breiman1984classification,meier2009high}. To achieve fast and accurate solutions, we propose a \emph{non-convex} projected gradient descent algorithm that consists of two stages: First, to encourage smooth $f_j$'s and capture relationships between neighboring feature quantiles, we employ non-convex piecewise-constant and piecewise-linear approximations. Secondly, we apply iterative \textit{group hard-thresholding} for sparse feature selection. The advantage of our formulation is that it decouples smoothness and sparsity, which leads to a simple algorithm that can be carried out with any sparse smoother and scales easily to high dimensions. Indeed, the proposed algorithm is low-complexity and converges quickly. To the best of our knowledge, our work is the only greedy non-convex algorithm with provable statistical and computational guarantees for learning sparse additive models~\cite{ravikumar2007spam}.

%Classically, sparsity is encouraged by $\ell_1$ penalty and smoothness is encoura  Meier et al. \cite{meier2009high}The related work in this direction include  to belong to a reasonable {\bf{Related work:}} Our work is within the additive model
% Our contributions

%. For enforcing   The advantages of our algorithm are three-folds. First, the function space we learn is more expressive than traditional linear models thanks to feature binning. Secondly, t

% and feature dependence is a one dimensional curve

{\bf{Contributions.}} This work provides both algorithmic and theoretical contributions to high-dimensional learning techniques. First, we develop the Non-Convex \algo,  which is significantly more expressive than traditional linear models and can efficiently learn feature non-linearities. The resulting algorithm is easier to understand and visualize, compared to regression trees and neural networks, as the overall decision function is separable over features. The algorithm is based on sparse matrix multiplication followed by fast projections; hence, runtime is competitive with iterative hard thresholding (IHT). Indeed, real and synthetic experiments complement our theoretical results and demonstrate that the proposed algorithm is competitive with gradient boosting. We provide computational and statistical guarantees on the rate of convergence and on the statistical precision of the proposed algorithm. In particular, for a random design data matrix $\X$, the algorithm converges \textit{linearly} to the optimal solution and requires near-optimal (minimal) sample complexity. %Hence, the proposed algorithm is competitive with gradient boosting while having provable guarantees as strong as lasso/IHT.

On the theory side, we provide a novel result for the convergence of non-convex projected gradient descent. To apply this result, we analyze the \bdata~$\Xbin$ derived from $\X$ and study its restricted eigenvalue conditions. Our analysis of this special random matrix is the key to fast convergence rates. We also illustrate an interesting connection between $\Xbin$ and sparse Johnson-Lindenstrauss matrices \cite{kane2014sparser,bourgain2015toward} which suggests improved dimensionality reduction techniques.

\subsection{{Related work}} It is often desirable to ensure gain functions have good properties, such as smoothness. The combination of sparse feature selection and smoothness is applied in additive models \cite{meier2009high,ravikumar2007spam,raskutti2012minimax} as well as fused Lasso \cite{tibshirani2005sparsity}. These works are based on convex optimization and use $\ell_1$ and total variation penalizations for regularization as well as splines \cite{meier2009high} and smooth basis functions \cite{ravikumar2007spam}. While there exists interesting statistical estimation results for additive models, they don't provide the deeper understanding we have for simpler linear models such as Lasso. For instance, it is unclear how much data we need for guaranteed training, and what is the convergence rate of iterative methods. Our algorithm is closely related to non-convex projections; in particular, iterative hard thresholding. In this direction, several works \cite{jain2014iterative,blumensath2009iterative} provide guaranteed convergence rates for classical sparse estimation problem as well as low-rank regression. Finally, our framework is inherently related to the decision/regression trees \cite{breiman1984classification} where the trees can learn nonlinear decisions while allowing for feature interactions (i.e. trees are multivariate functions unlike \eqref{additive model}). In fact, our algorithm with piecewise-constant projections corresponds to training an additive regression tree model where each tree uses a single feature.

{\bf{Notation.}}~We adopt the following notation throughout the paper. We use bold face letters for matrices and vectors. The transpose of a matrix $\A$ is denoted by $\A^{\top}$. $\Pc_\Cc(\cdot)$ is the projection operator on the set $\Cc$ and $\Cc-\Cc$ is the Minkowski difference of set $\Cc$. $\onebb$ denotes the all ones vectors of appropriate size and $\oneb$ denotes the indicator function. Given $\{b_i\}_{i\leq p}$ and $\pb=\sum_i b_i$, and a matrix $\A$ in $\R^{n\times \pb}$, $\A^j$ denotes the $j$th submatrix of size $n\times b_j$ so that $\A=[\A_1~\A_2~\dots~\A_p]$. For a vector $\vb\in\R^{\pb}$, $\vb^j\in\R^{b_j}$ is defined similarly. We use $\nnz{\cdot}$ to denote the number of nonzero entries of a vector or matrix. Finally, $\X_{:,i}$ and $\X_{i,:}$ denote the $i$th column and row of a matrix, respectively.

\section{Non-Convex Regularized Binned Regression}
Consider a response vector  $\y \in \mathbb{R}^{n}$ and a data matrix  $\X \in \mathbb{R}^{n \times p}$. We are concerned with the problem of modeling the dependent variable $\y$ as a linear combination of unknown \textit{functions} of individual features. In particular, for least-squares loss, we are interested in learning an estimator of the form \eqref{additive model}.
% samet look here Should we talk about a general loss instead?
%\begin{align}
%\hat{y}_i=\sum_{j\leq p}f_j(\X_{i,j}).\label{additive model}
%\end{align}
This class of estimators can be advantageous  to linear models such as Lasso which assign a scalar weight to each feature $j$.

To learn $f_j$'s, we will focus on {piecewise-constant} or {piecewise-linear} approximations (see Figure \ref{fig1}); however, our arguments can also be extended more generally. We propose the \algo~(RBR) algorithm which quickly learns these functions and  demonstrate that RBR has provably good statistical and computational properties.  We begin our development by presenting the two building blocks of the RBR algorithm: feature binning and linear estimation with decoupled sparsity and smoothness regularizers utilizing a non-convex projected gradient descent algorithm.
\begin{figure}
~~~ \begin{subfigure}[b]{0.4\textwidth}
        \includegraphics[width=\textwidth]{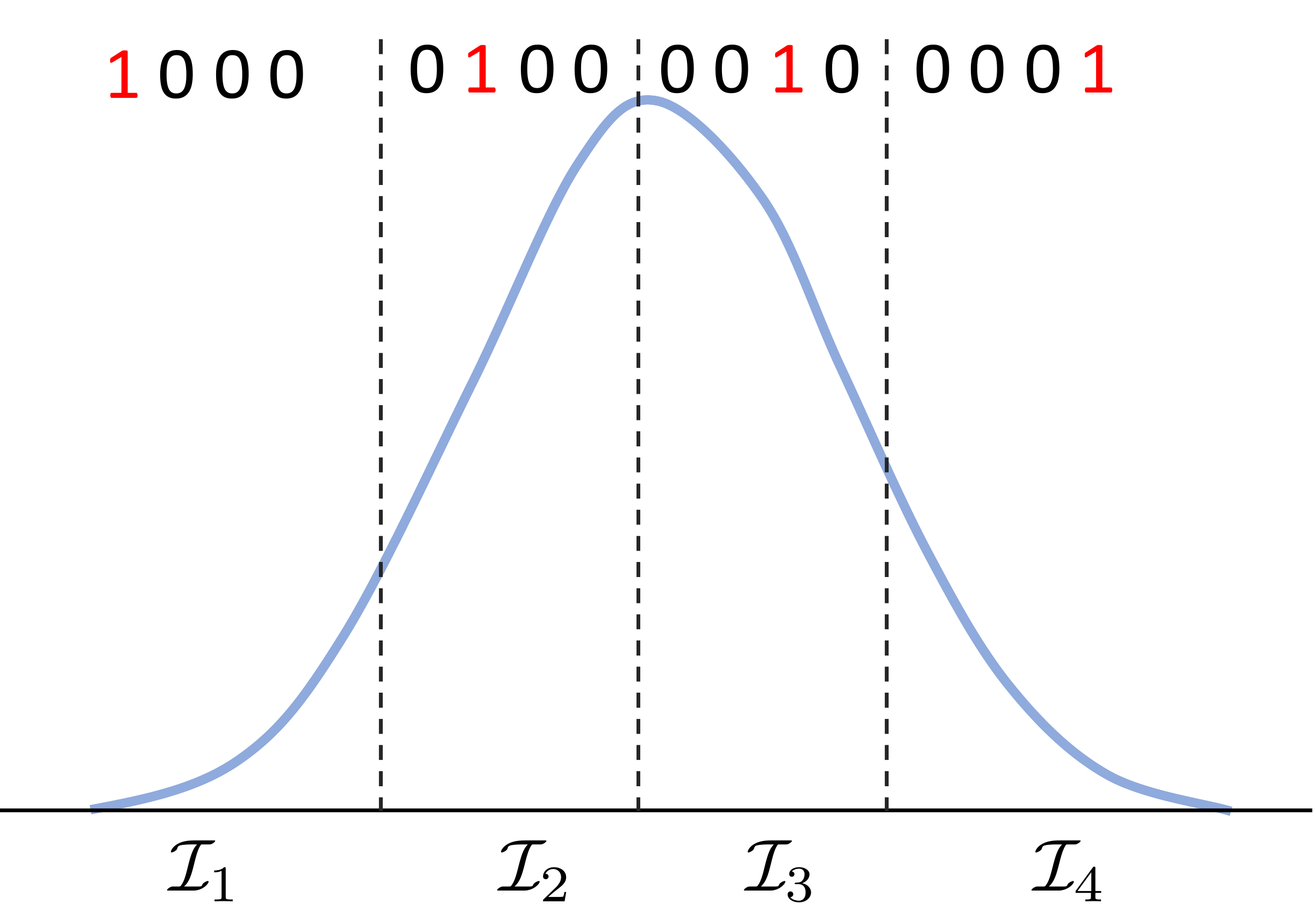}
        \caption{One-hot encoding feature quantiles given histogram of values.}
        \label{fig:binning}
    \end{subfigure}~~~~~~
\begin{subfigure}[b]{0.5\textwidth}
        \includegraphics[width=\textwidth]{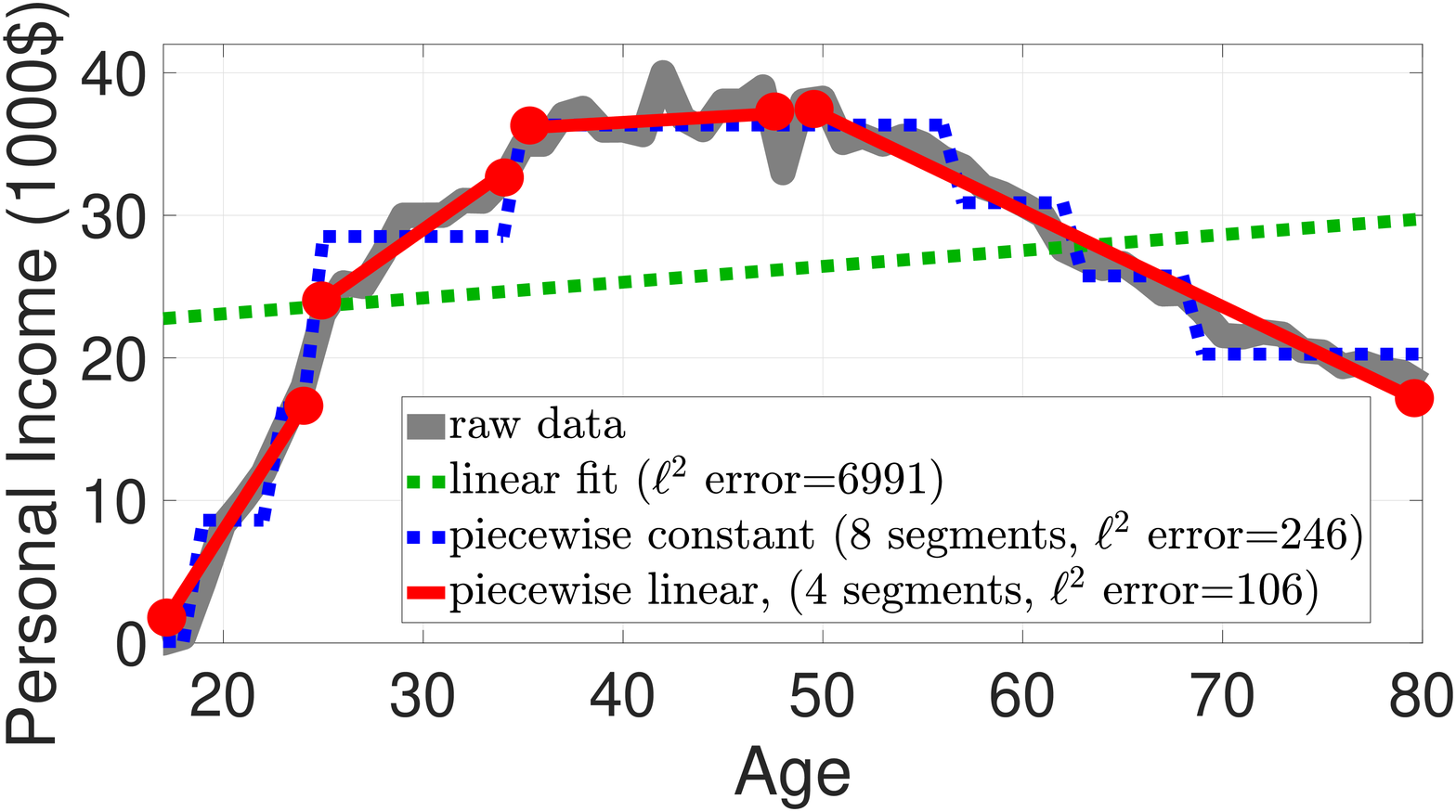}
        \caption{Personal income vs age, fitted by linear, piecewise constant, and piecewise linear functions.\cite{umn-data}}
        \label{fig1}
    \end{subfigure}
	\caption{}\vspace{-10pt}
%    \caption{Feature binning strategy and approximation of a function and Data from \cite{umn-data}}
\end{figure}

%\begin{figure}
%  \centering
%    \includegraphics[width=0.5\textwidth]{age_income}
%      \caption{}\label{fig1}
%\end{figure}

%The idea is to represent continuous features in a onehot encoded form. 
{\bf{Feature binning.}}~In order to accurately estimate $\{f_j\}_{j\leq p}$'s, we make use of the feature binning strategy. The idea is to map the data matrix $\X\in\R^{n\times p}$ to a larger one-hot encoded data matrix $\Xbin\in\R^{n\times \pb}$. Here $\pb/p$ corresponds to the average number of bins per feature. For each feature $1\leq j\leq p$, we split the real line into $b_j$ disjoint intervals $\{\Ic_{j,k}\}_{k\leq b_j}$ where $\Ic_{j,1}$ is the left-most and $\Ic_{j,b_j}$ is the right-most interval (see Figure \ref{fig:binning}). We then map $\X_{i,j}$ to the $k$'th standard basis vector $\e_k\in \R^{b_j}$ iff $\X_{i,j}\in \Ic_{j,k}$. This mapping also maps the $j$th feature column $\X_{:,j}$ to a matrix $\Xbin^j$ of size $n\times b_j$. Our data-dependent binning strategies are outlined in Definition \ref{strategy}. We obtain the \bdata~$\Xbin$ by concatenating the binning matrices for individual features as:
\[
\Xbin=[\Xbin^1~\dots~\Xbin^p] \in \mathbb{R}^{n \times \pb},
\]
where $\pb=\sum_{i\leq p}b_i$. The new estimator for $\y$ can then take the form $\hat\y = \Xbin\bbeta$ where $\bbeta\in\R^{\pb}$. This estimator helps us learn distinct coefficients for individual bins $\Ic_{j,k}$ as we learn a $b_j$ dimensional vector $\bbeta^j$ for feature $j$. The vector $\bbeta^j$ is essentially the discrete representation of the function $f_j$. %In addition to its computational virtues, we shall demonstrate in Section~\ref{sec-main-results} that the binning strategy also enjoys nice statistical guarantees. 

%With the help of this discretization, in SecOur main contribution is that we propose a projected gradient descent algorithm to learn a regularized $\bbeta$ and prove that it learns the feature nonlinearities quickly.

{\bf{Regularization with global sparsity and local smoothness.}}~To quickly and provably learn $\bbeta$ from labels $\y$ and \bdata~$\Xbin$, we apply two regularizations on $\bbeta$ which we call global and local constraints respectively. Global constraint enforces sparse feature selection whereas local constraint enforces smooth structure of individual feature vectors $\{\bbeta^j\}_{j\leq p}$:
\vspace{-0.1cm}
\begin{itemize}
\item {\bf{Global constraint:}} Select only $s$ features out of $p$, which implies $\bbeta$ is $s_G$-group sparse where group sparsity is defined as $\|\bbeta\|_{G}=\sum_{i=1}^p \oneb_{\bbeta^i\neq 0}$
\item {\bf{Local constraints:}} Each subvector $\bbeta^j$ satisfies a smoothness constraint $\|\bbeta^j\|_{L}\leq s_L$. In particular, we enforce $\bbeta^j$ to be composed of $s_L$ \emph{piecewise-constant} or \emph{piecewise-linear} segments as illustrated in Figure \ref{fig1}. For instance, piecewise-constant smoothness of a vector $\ab\in\R^d$ is defined as $\|\ab\|_L:=\|\ab\|_{PC}:=\sum_{i=1}^{d-1} \oneb_{\ab_{i}\neq \ab_{i+1}}$.
\end{itemize}
Observe that columns of each submatrix $\Xbin^i$ adds up to the all ones vector; hence, columns of $\Xbin$ are collinear. To work around this collinearity, we also enforce the condition $\onebb^{\top}\bbeta_j=0$.

To minimize a loss function $\Lc(\bbeta)$ (which is often $\Lc(\y,\Xbin\bbeta)$) for labels vector $\y$ and \bdata~$\Xbin$, we have the following optimization problem
\[
\min_{\bbeta} \Lc(\bbeta)~~~\text{subject to}~~~\|\bbeta\|_G\leq s_G,~\|\bbeta^j\|_L\leq s_L,~\onebb^{\top}\bbeta^j=0~\text{for}~1\leq j\leq p.
\]
We propose Algorithm \ref{pgd-algo}, which is a variant of projected gradient descent  to solve the optimization problem. The algorithm first takes a step in the direction of the gradient and then projects the updated solution $\bbeta_t$ to the constraint set. Observe that when dealing with a regression problem, the gradient takes the form $\nabla \Lc(\bbeta)=\Xbin^{\top}(\Xbin\bbeta-\y)$.  The projection step can performed efficiently by performing piecewise constant/linear approximation algorithms of~\cite{keogh2001locally}.

\begin{algorithm}[t] \hrulefill \\{\bf{Input:}} Sparsity $s_G$, local smoothness $s_L$, loss function $\Lc$, step size $\mu$, iteration count $\tau$.\\
{\bf{Output:}} Parameter $\bbeta$\\
{\bf{Initialization:}} $\bbeta_0\gets 0$, $t=0$.\\
\While{$1\leq  t\leq  \tau$}{
$\bbeta_{t}\gets \bbeta_{t-1}-\mu\grad{\bbeta_{t-1}}$\hspace{48pt}(gradient descent)\\
\For{$1\leq i\leq p$}{
$\bbeta^j_{t} \gets \bbeta^j_t-b_j^{-1}\onebb\onebb^{\top}\bbeta^j_t$\hspace{33pt}~~~~~(local zero-mean)\\
$\bbeta^j_{t} \gets \Pc_{L}(\bbeta^j_t)$\hspace{75pt}($s_L$ local-smoothness with projection operator $\Pc_{L}$)\\
}
$\bbeta_{t}\gets\Pc_{G}(\bbeta_t)$\hspace{92pt}($s_G$ global-sparsity with projection operator $\Pc_{G}$)\\
}
$\bbeta\gets\bbeta_\tau$
\caption{Non-Convex Regularized Binned Regression (RBR)}\label{pgd-algo}\vspace{-3pt}
\end{algorithm}
\vspace{-2pt}
Before moving to our theoretical contributions, we outline the advantages of proposed RBR algorithm.
\begin{itemize}
\item The proposed algorithm can learn inherent nonlinearities of features and is more expressive than linear models such as Lasso.
\item While $\Xbin$ is a larger matrix, it is sparse (in particular, $\nnz{\Xbin}=\nnz{\X}$), and gradient iterations involve sparse matrix multiplication. Hence, the runtime is fast (see Section \ref{sec runtime}).
\item The output is interpretable because the final prediction $\Xbin\bbeta_\star$ is separable in individual features as $\sum_{i=1}^p \Xbin^j\bbeta_\star^j$. Each feature can be visualized by plotting $\bbeta_\star^j$.
\item The algorithm is based on gradient descent, and can possibly be integrated with neural networks. While there are technical challenges, one could replace the softmax layer of deep neural networks with Algorithm \ref{pgd-algo}. The additional expressivity might help with faster convergence or improved prediction performance.
\end{itemize}
\vspace{-5pt}
% We are interested in ensuring that Algorithm $1$ converges for a given loss function. We first present a result on convergence of projected gradient descent for nonconvex loss functions. This result is in a similar flavor to that of \cite{} however, we allow arbitrary projections. In return, our assumption is stronger than restricted strong convexity.

\section{Main results}\label{sec-main-results}
In this section, we answer fundamental questions about the statistical and computational properties of the RBR algorithm. In particular, we aim to rigorously answer the following key questions:
\begin{itemize}
\item Under what conditions does the proposed RBR algorithm succeed?
\item Does the design matrix $\Xbin$ generated by feature binning have desirable statistical properties for high-dimensional learning problems?
\end{itemize}
%First, we shall establish the conditions for the convergence of Algorithm \ref{pgd-algo}. 
Let us denote the overall constraint set (combination of global and local constraints) as $\Cc$, defined as:
\begin{align}
\Cc=\{\bbeta\in\R^{\pb}~\big|~\|\bbeta\|_G\leq s_G,~\|\bbeta^j\|_L\leq s_L,~\onebb^{\top}\bbeta^j=0~1\leq j\leq p\}\label{constraint_eq}
\end{align}
%let us denote the overall projection operator  as $\Pc_{\Cc}(\cdot)$ where $\Cc$ is the set $\bbeta$ belongs. 
We will analyze general non-convex projections $\Pc_{\Cc}(\cdot)$ where the main property we use is that $\Cc$ is a closed cone. We introduce \emph{restricted gradient correlation} (RGC) for analyzing such projections.
\begin{definition}[Restricted gradient correlation condition] \label{def rgc} Function $\Lc(\cdot)$ obeys RGC over the set $\Cc$ with parameters $\mur,\epsr>0$ if all vectors $\vb\in\Cc-\Cc$ and $\x,\y\in\Cc$ satisfy
\[
|\li\vb,\x-\y\ri-\mur\li\vb, \grad{\x}-\grad{\y}\ri|\leq \epsr\tn{\vb}\tn{\x-\y}.
\]
\end{definition}
%samet look here
%It can be shown that RGC implies conventional restricted smothness and restricted strong convexity conditions \cite{jain2014iterative,negahban2009unified}. Furthermore, it is implied by them if $\grad{}$ is a linear operator. 
We note that the RGC condition is closely related to the restricted smoothness and restricted strong convexity conditions that find frequent use in high-dimensional problems~\cite{jain2014iterative,negahban2009unified}. For instance, RGC is implied by them if $\grad{\cdot}$ is a linear operator. 

\subsection{Convergence analysis and statistical guarantees}
We now demonstrate that the loss function converges \textit{geometrically} with restricted gradient correlation and that one can achieve good statistical estimation with RBR.  First, we state a general result on the convergence of  Algorithm \ref{pgd-algo} and   then establish the conditions for the convergence of RBR. Our  result is in a similar flavor to \cite{jain2014iterative}; however, \cite{jain2014iterative} applies to general non-convex sets $\Cc$ to address \eqref{constraint_eq} instead of only sparsity constraints.  We refer the reader to the supplementary material for the detailed proofs of all of our results.
\begin{theorem}\label{general loss} Suppose $\Lc(\cdot)$ obeys restricted gradient correlation with parameters $\mur,\epsr$ over the set $\Cc$ and let $\bbetas= \arg\min_{\bbeta \in \Cc} \Lc(\bbeta)$ be the  {\emph{unique}} minimizer. Starting from $\bbeta_0=0$, run the iterations
\[
\bbeta_{t+1}=\Pc_\Cc(\bbeta_t-\mur\grad{\bbeta_t})%Algorithm \ref{pgd-algo} for $\tau$ steps
\] for $\tau$ steps. When $\epsr<0.5$, we have linear convergence to $\bbetas$ as follows
\begin{align}
\tn{\bbeta_{\tau}-\bbetas}\quad&\leq\quad (2\epsr)^\tau\tn{\bbetas},\nn\\
\Lc(\bbeta_\tau)-\Lc(\bbetas)\quad&\leq\quad(2\epsr)^{2\tau}\frac{\epsr+1}{\mur}\tn{\bbetas}^2+(2\epsr)^\tau \tn{\bbetas}\tn{\grad{\bbetas}}.\nn
\end{align}
Furthermore, for any $\bbeta \in \Cc$ that is  estimated by minimizing $\Lc$, we have 
\[
\tn{\bbeta_{\tau}-\bbeta}\leq (2\epsr)^\tau\tn{\bbeta}+\frac{2\mur}{1-2\epsr}\tn{\Pc_{\Cc-\Cc}(-\grad{\bbeta})}.
\]
\end{theorem}
We now demonstrate the implications of Theorem \ref{general loss} on Algorithm \ref{pgd-algo}. To do this, we need to understand the properties of binned matrix $\Xbin$, which depends on how the features are binned. We analyze two different strategies for binning data matrix:
\begin{definition} [Feature binning schemes] \label{strategy}For each $1\leq j\leq p$, the nonzero entries over $\Xbin^j$ are equal to $\sqrt{b_j/n}$. Furthermore,%We have the following definitions for feature binning.
\begin{itemize}
\item {\bf{\rrb:}} Allocates equal number of samples ($n/b_j$) to each bin for all features.
%$\Xbin_R\in\R^{n\times \pb}$ is chosen uniformly at random from the set of matrices satisfying: i) Columns of $\Xbin^j$ have $n/b_j$ nonzero entries. ii) Nonzero entries of $\Xbin^j$ are $\sqrt{b_j/n}$. This scheme corresponds to allocating equal number of samples ($n/b_j$) each bin each feature. %Allocates equal number of samples ($n/b_j$) to each bin for all features.
\item {\bf{\irb:}} For each feature $j$, the size of each bin follows a binomial distribution with parameters $(n,1/b_j)$; i.e.,~the sum of $n$ independent Bernoulli's with mean $b_j^{-1}$. The sole dependence of $b_j$ bin size variables is that they add up to $n$. % The sole dependence of bin sizes is that total adds up to $n$.
%$\Xbin_I\in\R^{n\times \pb}$ is a matrix with independent rows. Each row is a random encoding vector normalized by $\sqrt{n}$. This scheme corresponds to binning $j$th feature based on the number of nonzeros of columns of $\Xbin^j$.
%chosen uniformly at random from the set of matrices satisfying the followings: 1) $\{\Xbin^j\}_{j=1}^p$ are independent random matrices. 2) 
%$\Xbin\in\R^{n\times \pb}$ is a random matrix where each row is statistically identical to a random encoding vector  normalized by $\sqrt{n}$ and it is generated uniformly at random from matrices having equal number of nonzeros at each block $\Xbin^j$ for $1\leq j\leq p$. This corresponds to allocating $n/b_j$ samples each bin each feature.
%For each feature, create 
\end{itemize}
\end{definition} 
Observe that for fixed $\{b_j\}_{j\leq p}$, regular and binomial binning methods have similar bin sizes as $n\rightarrow\infty$ due to law of large numbers. % With these definitions 
To state the main result on the performance of the algorithm, we require the following assumptions on data matrix $\X$.
\begin{assumption}\label{main assume} Entries of $\X$ are random variables with continuous distribution. Furthermore,
\begin{itemize}
\item {\bf{Independent features:}} $\X$ has independent columns.
\item {\bf{Independent identical samples:}} $\X$ has independent identically distributed (i.i.d.) rows. 
\end{itemize}
\end{assumption}
% with independent rows each of which are random encoding vectors. This corresponds to binning $j$th feature based on the number of nonzeros of columns of $\Xbin^j$. 
These assumptions are sufficient to ensure that each row of $\Xbin$ is composed of randomly one-hot encoded vectors: Each row is statistically identical to a vector $\ab\in\R^{\pb}$ where $\{\ab^j\}_{j=1}^p$ are $b_j$-dimensional independent vectors and $\sqrt{n/b_j}\ab^j$ is uniformly distributed over the standard basis. Furthermore, with the \irb~strategy, $\Xbin$ has \emph{i.i.d.} rows, which is crucial for our analysis. Continuous distribution assumption is used to ensure that feature values are distinct with probability $1$ and there is no ambiguity during binning stage.

By construction, for any $\bbeta\in\Cc$, $\Xbin$ obeys $\E[\Xbin\bbeta]=0$ and $\E[\tn{\Xbin\bbeta}^2]=1$. We have the following result for quadratic loss function $\Lc(\bbeta)=\tn{\y-\Xbin\bbeta}^2$ when using Algorithm \ref{pgd-algo}.% when features are independent of each other and ``\irb`` strategy is employed.
\begin{theorem}\label{main binning} Suppose $\max_{1\leq j\leq p} b_j$ is upper bounded by a constant and Assumption \ref{main assume} holds. Create a {\em{\irb}}~matrix $\Xbin$.
%Let features of data matrix $\X$ be binned according to {\bf{\irb}} strategy to form binned matrix $\Xbin$ so that $\Xbin$ has independent rows which are random feature encodings. 
Suppose we observe samples $\y\in\R^n$ obeying
\[
\y=\Xbin\btrue+\z
\]
for some planted vector $\btrue\in\R^{\pb}$ and noise $\z\in\R^n$. There exists constants $c,C>0$ such that if \begin{align}n>n_0:=cs_Ls_G\log \pb\label{sample comp}\end{align} with probability $1-\exp(-Cn)$, starting from $\bbeta_0=0$, the iterations of Algorithm \ref{pgd-algo} with step size $\mu=1$ obey
\begin{align}
\tn{\bbeta_\tau-\btrue}\leq \left(\frac{n_0}{n}\right)^{\tau/2}\tn{\btrue}+\eta\tn{\z}\label{opt error}
\end{align}
with $\eta=\sqrt{\frac{n_0}{n}}$. Under the same assumptions, if we employ a {\em{\rrb}}~matrix $\Xbin$, \eqref{opt error} holds for $\eta=2$, with probability $1-\exp(-Cn)-\pb^{-10}$ as long as we additionally have $n\geq cs_G^2\max_{j\leq p}b_j\log\pb$.
\end{theorem}
%\proofsk{}
\begin{proof} We only provide a sketch of proof and defer the detailed proof to the appendix. The proof of first claim consists of two steps. The first step involves properties of the \irb~matrix $\Xbin$ when $\X$ obeys Assumption \ref{main assume}. By construction $\Xbin$ have independent rows but dependent columns. Denote its first row by $\xbin$. Consider the subspace $S_{zm}=\{\vb\in\R^{\pb}\big|\onebb^{\top}\vb^i=0~1\leq i\leq p\}$. Our first major result shows that for any $\vb\in S_{zm}$, $\xbin^{\top}\vb$ is a zero-mean subgaussian random variable with unit variance. From the results of \cite{companion, UUP2}, restricted gradient correlation can be controlled if size of the constraint set $\Cc$ is small. We control size in terms of Gaussian complexity \cite{McCoy,Cha} which is defined as $\omega(\Cc)=\E[\sup_{\vb\in\Cc,\tn{\vb}\leq 1}(\g^{\top}\vb)^2]$. In particular, $\epsr^2\lesssim \order{\frac{\omega(\Cc)}{n}}$.

Our second crucial estimate is bounding the quantity $\omega(\Cc)$ in terms of sparsity $s_G$ and smoothness $s_L$. We do this by proving the upper bound $\omega(\Cc)\leq c s_Gs_L\log \pb$.

The proof of our result on \rrb~is based on approximating a \rrb~matrix in terms of a \irb~matrix and is provided in the supplementary material.
% The idea is to bound restricted eigenvalues of $\Xbin_R$ in terms of that of $\Xbin_I$ and a small perturbation matrix.
\end{proof}
Remarkably, for the \irb~scheme, Theorem \ref{main binning} is optimal in the sense that the number of required samples $n_0$ is proportional to the degrees of freedom $s_Ls_G$ (total number of discontinuities of $\btrue$) up to logarithmic factors. Furthermore, the error bounds provided in \eqref{opt error} is consistent with state-of-the-art results such as \cite{companion,jain2014iterative,negahban2009unified}.

For the more practical \rrb~scheme, we have similar but weaker results. This is due to the fact that, binning matrix $\Xbin$ has more structure (e.g. dependent rows) and is more challenging to analyze. In particular, the number of selected features $s_G$ can scale as $\order{\sqrt{n}}$ instead of $\order{n}$.

%In Theorem \ref{main binning}, \irb~strategy provides state of the art estimation and convergence bounds. However, it has the drawback that, features should be binned according to the \irb~matrix described in Definition \ref{strategy}. A more practical approach is equally partitioning a feature to each bin which results in the \rrb~matrix $\Xbin=\Xbin_R$. This matrix is more challenging to analyze due to dependent rows hence we require the additional constraint \eqref{regular binning}.% We have the following corollary of Theorem \ref{main binning}.
%\begin{corollary}\label{cor binning} Consider the setup in Theorem \ref{main binning}. Let $\x\in\R^p$ be a vector with independent entries having continuous distribution. Form the data matrix $\X\in\R^{n\times p}$ by using $n$ i.i.d.~copies of $\x$. Let the binning intervals $\{I_{j,k}\}_{j\leq p,k\leq b_j}$ are so that $\Pr(\x_j\in I_{j,k})=b_j^{-1}$. Form binned matrix $\Xbin$ from $\X$. Suppose samples $\y\in\R^n$ obey
%\[
%\y=\Xbin\btrue+\z
%\]
%where $\btrue\in\R^{\pb}$ is the true hidden vector and $\z\in\R^n$ is a Gaussian noise. Then, there exists constants $c_1,c_2$ such that if $n>n_0:=c_1s_Ls_G\log \pb$, with probability $1-\exp(-c_2n)$, starting from $\bbeta_0=0$, the iterations obey
%\[
%\tn{\bbeta_\tau-\btrue}\leq \left(\frac{n_0}{n}\right)^{\tau/2}\tn{\btrue}+\sqrt{\frac{n_0}{n}}\tn{\z}.
%\]
%\end{corollary}

% This bound applies to not only piecewise constant smoothness, 
%\end{proof}
We should add the following remarks for better interpretation of Theorem \ref{main binning}:
\begin{itemize}
%\item Theorem \ref{main binning}, assumes knowledge of intervals $I_{j,k}$ that arises from true distribution of $\x_j$. In practice, $I_{j,k}$ are derived from empirical data $\X_{:,j}$ which is the $j$th feature column. before data matrix $\X$ is generated. In practice, a naive approach would be to distribute feature values equally to each bin. This would imply that each column of $\X^j$ contain $n/b_j$ nonzeros. In contrast, $\Xbin$ of Theorem \ref{main binning} has $n/b_j$ nonzeros in expectation. Theorem \ref{main binning} also applies to the following data dependent binning strategy. Generate a matrix $\Xbin$ by i.i.d. sampling $\xbin$. Then, for each feature bin the values so that, $k$th bin of $j$th feature has exactly $\|\Xbin^j_{:,k}\|_0$ samples. This binning is not uniform, but converges to uniform as $n$ grows.
\item Theorem \ref{main binning} can be generalized to account for distinct smoothness levels $\{s_{L,j}\}_{j\leq p}$. In \eqref{sample comp}, we simply replace $s_Ls_G$ with the sum of top $s_G$ elements of $\{s_{L,j}\}_{j\leq p}$ (see appendix).
\item It should be remarked that, our analysis also addresses one-hot encoded categorical (discrete) features. We can simply set $b_j$ to be the number of distinct feature values and let $s_{L,j}=b_j$.
\item The identical result applies to other types of smoothness such as piecewise-linear approximation where $s_L$ is the number of non-differentiable points instead of discontinuities.
%\item Observe that, by construction $\Xbin\bbeta$ has zero mean for all $\bbeta\in\Cc$. In practice, one can append an all-ones vector to form $\Xbin'=[\Xbin~\onebb]$ matrix and solve a $n\times (\pb+1)$ dimensional system that allows for the constant bias term in prediction.
\end{itemize}

\subsection{Runtime analysis}\label{sec runtime}
%An important feature of machine learning algorithms is the computational complexity. RBR algorithm enjoys desirable properties 
As stated in Theorem \ref{main binning}, the proposed RBR algorithm  converges linearly and requires $\log(1/\eps)$ iterations to achieve $\eps$ accuracy.  Here, we study the computational complexity of RBR and demonstrate that it enjoys desirable properties.  To find the overall runtime, we focus our analysis  on the running time of each step $\bbeta_{t+1}\gets\bbeta_t-\Lc(\y,\Xbin\bbeta_t)$. For most applications such as classification with logistic loss and regression with squared loss, which are special cases of generalized linear models (GLM) \cite{mccullagh1984generalized}, the gradient has the following form
\[
\grad{\y,\Xbin\bbeta}=\Xbin^{\top}(g(\Xbin\bbeta)-\y)
\]
where $g(\cdot)$ is the so-called mean function. For least-squares $g(\cdot)$ is the identity and for logistic regression, $g(a)=\exp(a)/(1+\exp(a))$. In both cases, $g$ can be calculated in $\order{1}$ time. Hence, the gradient step takes $\order{\nnz{\Xbin}}=\order{\nnz{\X}}\leq \order{np}$. The second step is the projection which involves:
\begin{itemize}
\item {\bf{Local zero-mean:}} $\order{\pb=\sum_{i}b_i}$ complexity,
\item {\bf{Local smoothness:}} For feature $j$, an $\order{b_j\log b_j}$ algorithm exists for good piecewise-constant approximation \cite{keogh2001locally}. Hence the total runtime is $\order{\pb\log \max_j b_j}$.%We remark that one can also minimize convexified smoothness total variation minimization as well \cite{condat2013direct}.
\item {\bf{Global sparsity:}} Requires thresholding vectors $\bbeta^j$ by their $\ell_2$ norms and returning the top $s_G$ entries. Consequently, it has $\order{\pb+p\log p}$ time complexity.
\end{itemize}
Overall, each iteration takes $\order{\pb\log \pb+np}$ time. As long as $\max_jb_j\log \pb<\order{n}$, the dominating term has the same complexity as doing the matrix-vector product $\X\bbeta$, which is needed for the gradient iteration of ordinary linear regression. The total time it takes to reach an $\eps$ accurate solution is $\order{\log \eps^{-1}(np+\pb\log \pb)}$ as long as the RGC is satisfied.

% which is summarized in the following lemma.
%\begin{lemma} Suppose $\Lc(\bbeta)=\Lc(\y,\Xbin\bbeta)$ is the least squares or logistic loss function. Each step $\bbeta_{t+1}\gets \Pc_\Cc(\bbeta_t-\mu\grad{\bbeta})$ of the algorithm requires 
%\end{lemma}

\section{Numerical results}
This section is dedicated to numerical experiments involving our algorithm. %As discussed previously, the function space that can be represented by our algorithm is a superset of lasso-type algorithms and subset of gradient boosted regression trees. As a result, we will compare the numerical performance of three algorithms:
We compare the numerical performance of the following algorithms:
\begin{itemize}
\item \algo~(RBR)~(Algorithm~\ref{pgd-algo}),
\item Iterative hard-thresholding (IHT) \cite{blumensath2009iterative},
\item Gradient boosted regression trees (GBRT) (XGBoost implementation \cite{chen2016xgboost})
\end{itemize}
%With infinite amount of data, we expect that the more expressive algorithm should perform the best. This implies that, \algo~will likely outperform linear models (such as IHT) and will be outperformed by GBRT.
%As mentioned earlier, for \algo, we use the piecewise constant/linear approximation algorithms of \cite{keogh2001locally} to perform the projection step. 
For all of our experiments, we use $s_L=8$ segments per feature and $b_j=40$ bins. For XGBoost, we use $10$ trees with maximum tree depth of $6$. Also for all classification tasks, the training phase uses 80\% of the data. % As 

\begin{figure}
 \begin{subfigure}[b]{0.32\textwidth}
        \includegraphics[width=\textwidth]{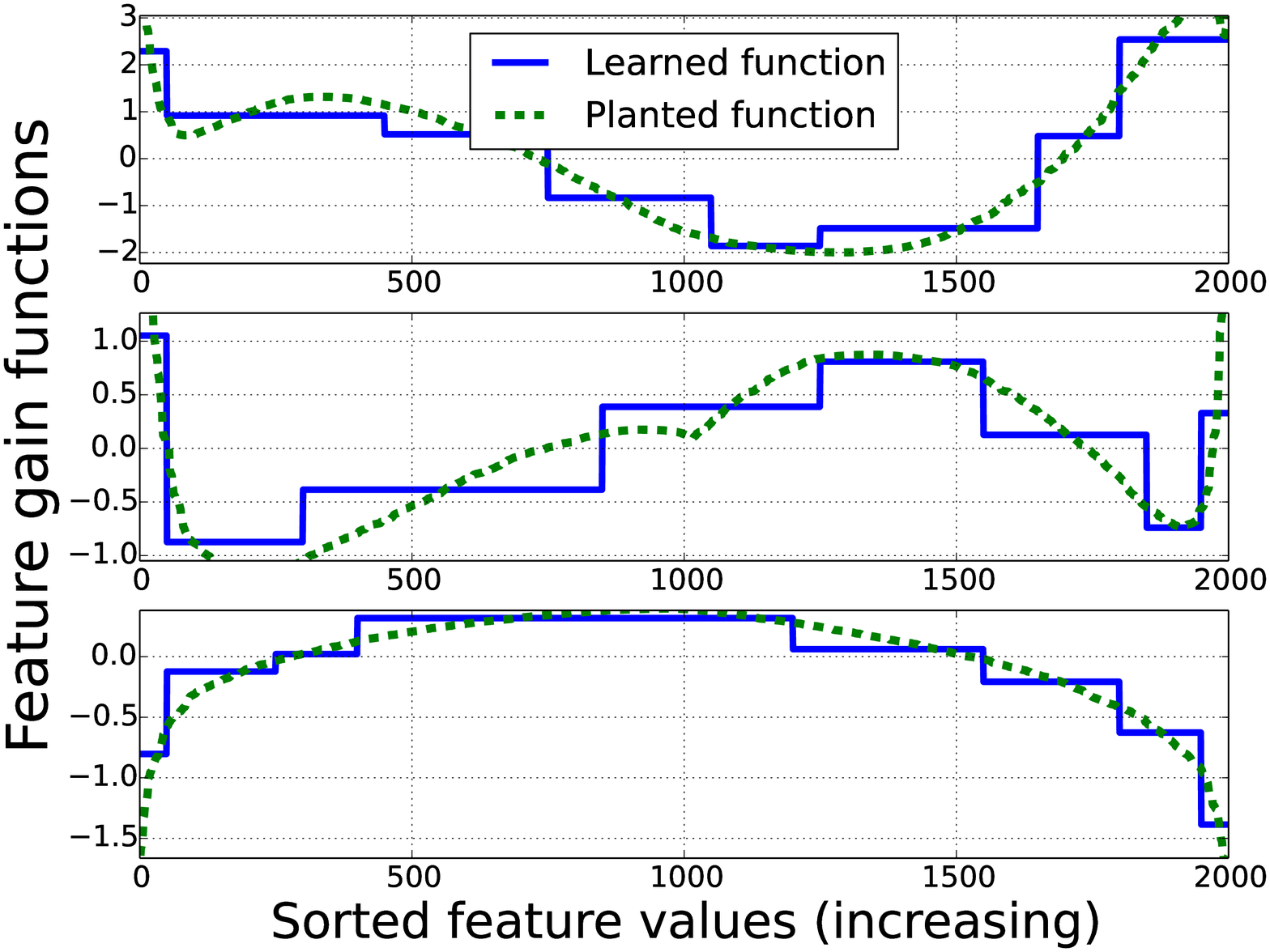}
        \caption{Learning planted non-linearities with Algorithm \ref{pgd-algo}}
        \label{fig:nonlinear}
    \end{subfigure}~
    \begin{subfigure}[b]{0.32\textwidth}
        \includegraphics[width=\textwidth]{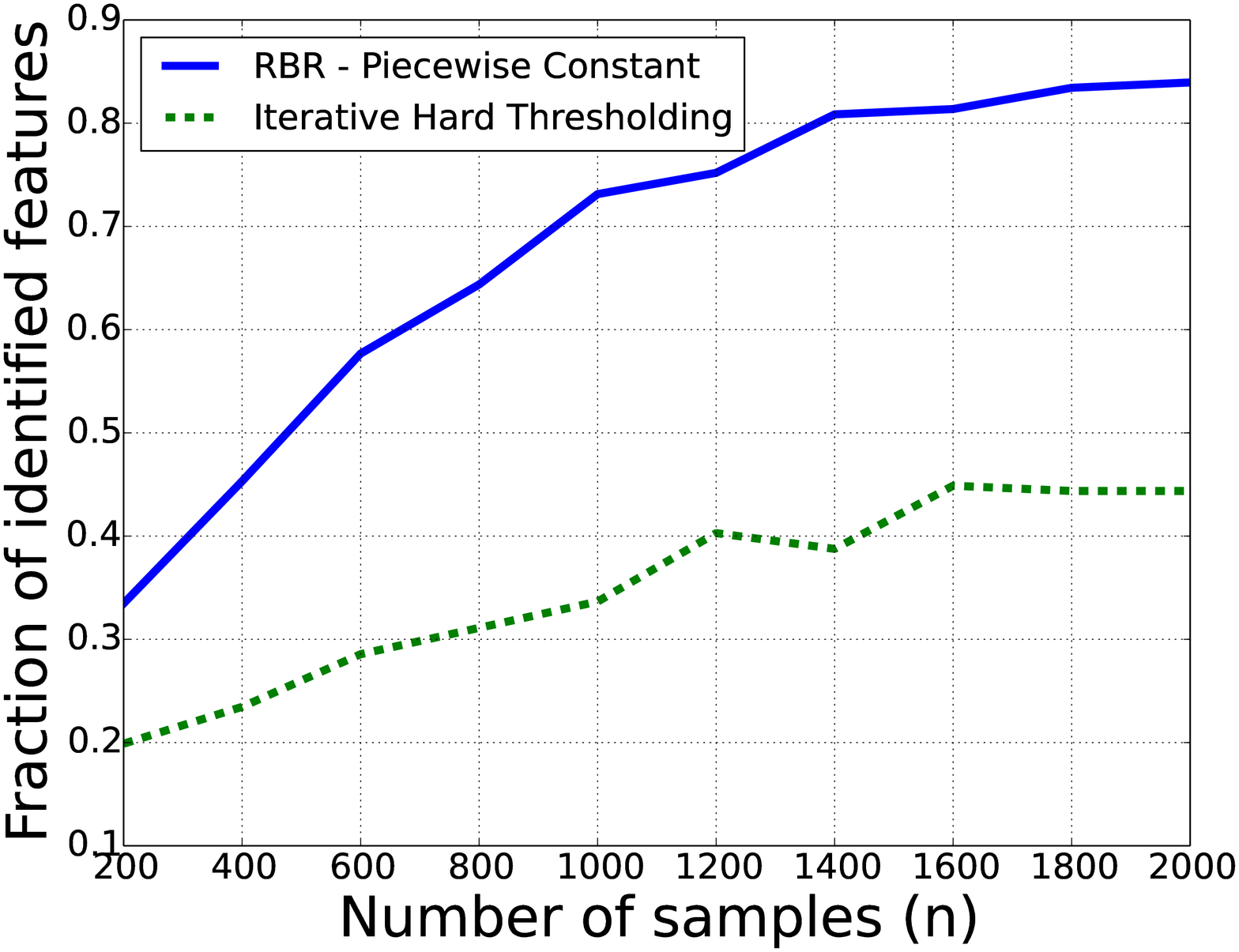}
        \caption{Feature selection performance of RBR vs IHT}
        \label{fig:select}
    \end{subfigure}~
    \begin{subfigure}[b]{0.35\textwidth}
        \includegraphics[width=\textwidth]{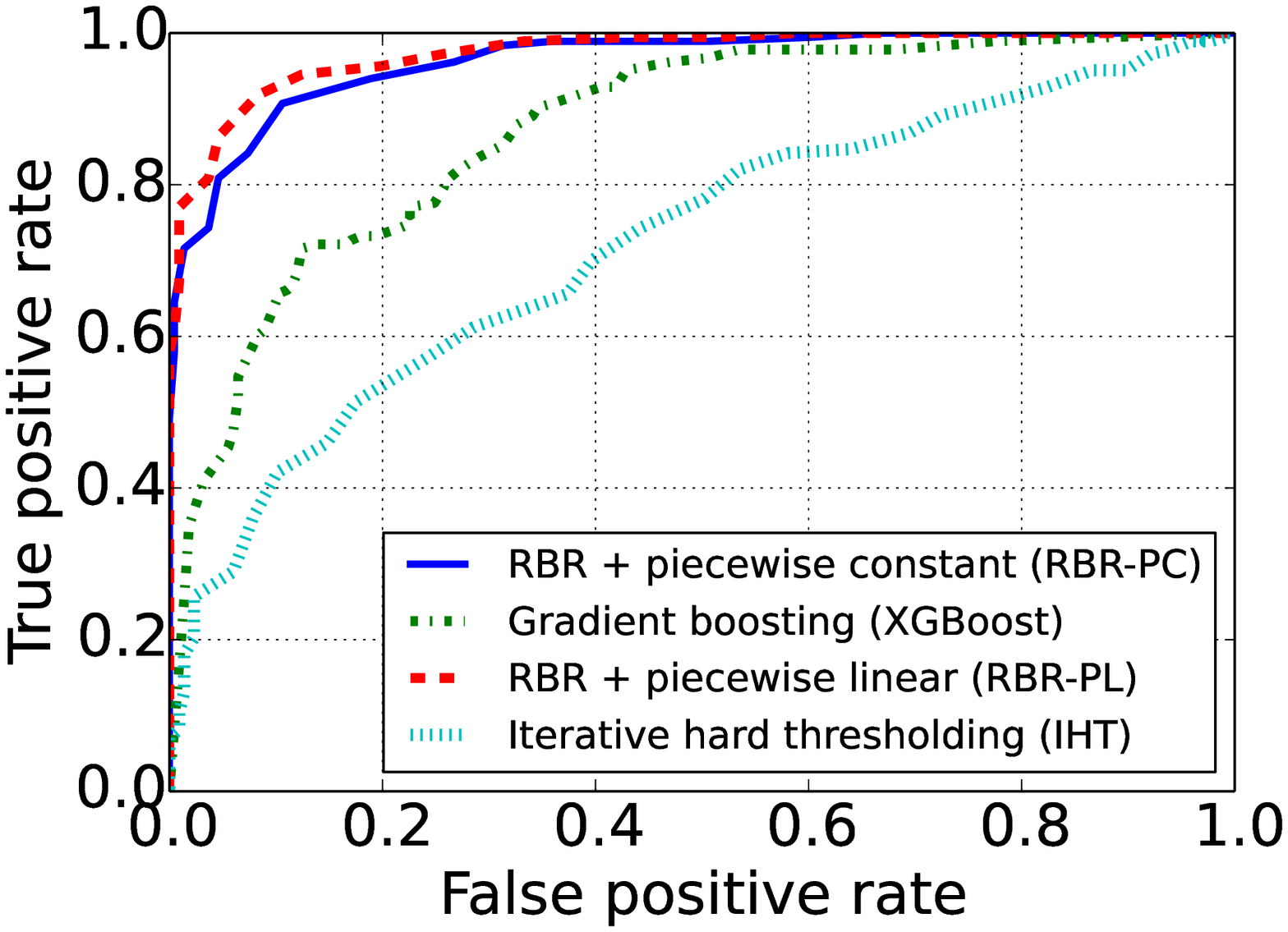}
        \caption{ROC curves for synthetic data}
        \label{fig:rocsynt}
    \end{subfigure}
    \caption{}
\end{figure}

\subsection{Results on synthetic data}
In order to test the performance of RBR, we first consider a synthetic dataset where we engineer a planted model for the feature non-linearities \eqref{additive model}. We generate $\X$ as a Gaussian data matrix with independent standard normal entries. In \eqref{additive model}, we pick planted functions of the form
\[
f_i(x)=\alpha_{i,1}|x|^{\alpha_{i,2}}+\alpha_{i,3}\sin(\alpha_{i,4}x+\alpha_{i,5})
\]
where $\alpha_{i,j}$'s are independent random variables either uniformly distributed ($j\in[2,5]$) or Gaussian distributed ($j\in[1,3,4]$) for $s_G$ out of $p$ features. We set $n=2000$, $p=100$, $\pb=40p=4000$, and plant $s_G=10$ nonzero features. We first test the ability to learn feature non-linearities. For a random problem instance, Figure \ref{fig:nonlinear} overlays the learned parameter $\bbeta^i$ with non-linear function $f_i$ for nonzero $f_i$'s. Here, the horizontal axis is the sorted feature indices, where every $50$ values corresponds to a new bin. We observe that $\bbeta^i$ indeed learns the non-linear function $f_i$ to a good extent. In Figure \ref{fig:select}, we contrast the feature selection performance of IHT and RBR~where both know the true sparsity level $s_G/p=0.1$. As a performance metric, we use the fraction of correctly identified nonzero features. Over $20$ random problem instances, we average this metric for IHT and RBR and plot the values for $n=200$ to $2000$. As the amount of data increases, both methods show improved performance, but RBR is generally better, and identifies $8$ out of $10$ features correctly on average for $n\geq 1400$.

Finally, in~Figure \ref{fig:rocsynt}, we use the same setup as Figure~\ref{fig:nonlinear} but convert labels to $\{0,1\}$ by thresholding at the median value. We train a classifier via logistic regression and compare RBR Piecewise-Constant, RBR-Piecewise-Linear, IHT and XGBoost over test set. IHT and RBR use the planted sparsity level $s_G/p=0.1$. The receiver operating characteristic (ROC) curve is plotted in Figure \ref{fig:rocsynt}. We can see that RBR-PC and RBR-PL substantially outperform the competing algorithms.
%\[
%f_i(x)=\alpha_{i,1}|x|^{2\alpha_{i,2}}+\alpha_{i,3}\sin(2\alpha_{i,4}x+\pi\alpha_{i,5})
%\]
%where $\alpha_{i,j}$'s are independent random variables where for $j\in [1,3,4]$ $\alpha_{i,j}$ is standard normal and for $j\in [2,5]$ $\alpha_{i,j}$ is uniform distribution between $[0,1]$.

\begin{figure}
\begin{subfigure}[b]{0.5\textwidth}
        \includegraphics[width=\textwidth]{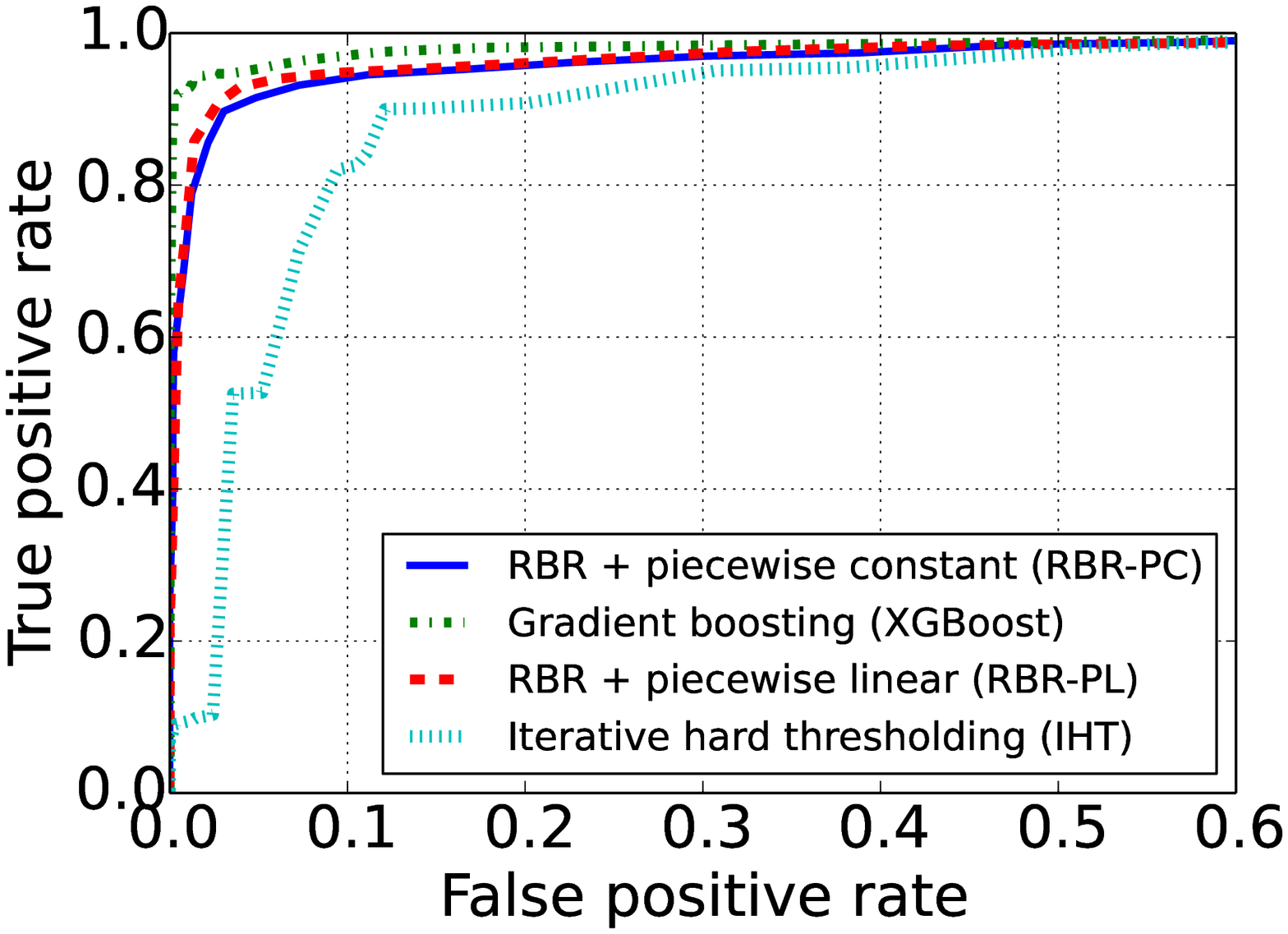}
        \caption{Human resources dataset}
        \label{fig:rochr}
    \end{subfigure}~
 \begin{subfigure}[b]{0.5\textwidth}
        \includegraphics[width=\textwidth]{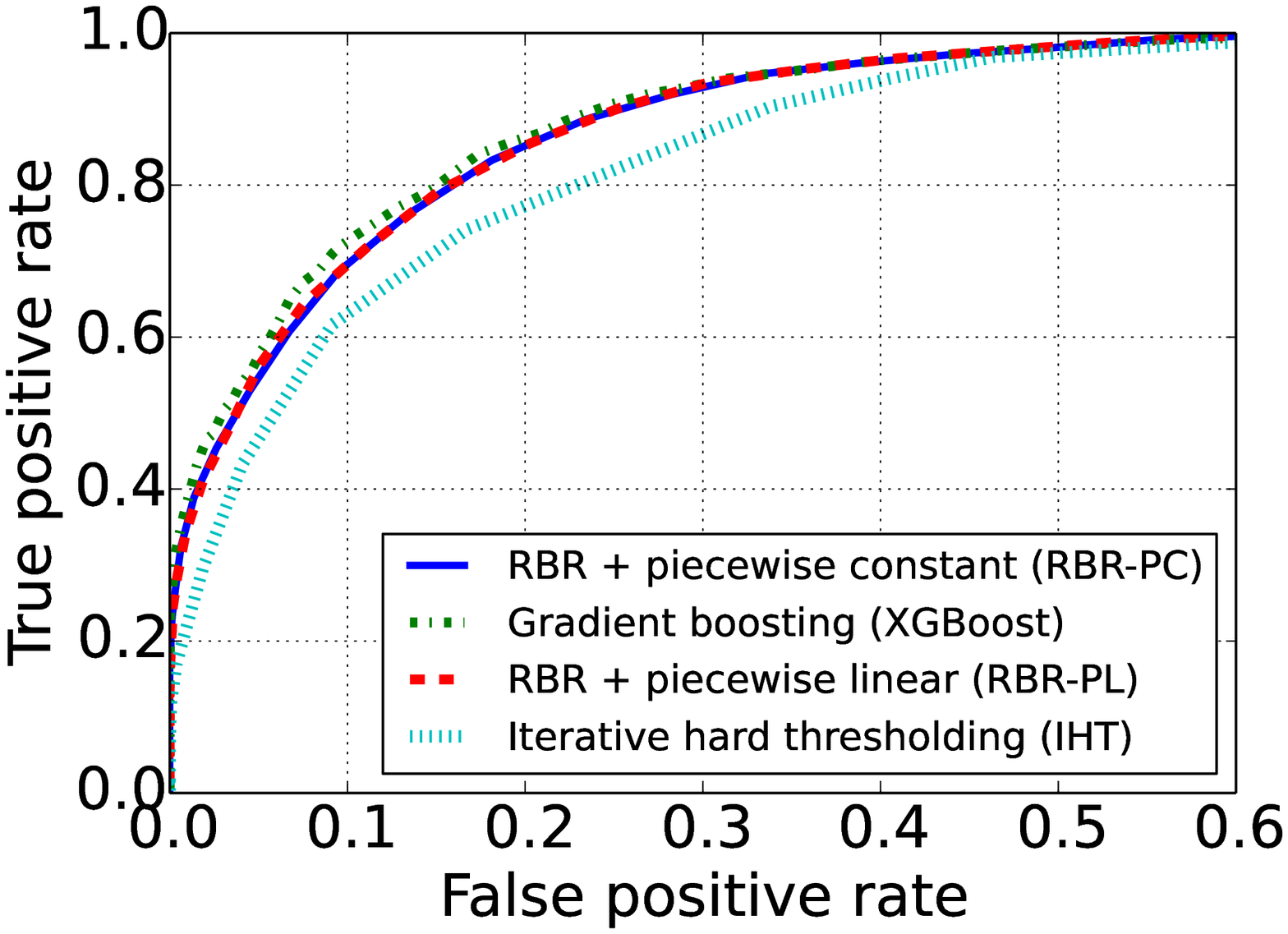}
        \caption{Adult dataset}
        \label{fig:rocadult}
    \end{subfigure}
   \caption{}
\end{figure}

\subsection{Results on real datasets}
We now compare the performance of RBR with IHT and XGBoost on two real datasets. The first dataset is a Human Resources Analytics dataset from Kaggle where the goal is predicting if an employee left the company \cite{kagglehr}. This dataset has three continuous and six categorical features and $n=15000$ samples. The second dataset is the Adult dataset on UCI Machine Learning repository. It is based on 1994 Census database, and the goal is to predict if income exceeds \$50K per year \cite{uciml}. There are five continuous and nine categorical features and $n=32560$ samples. For both datasets, we encode categorical features as one-hot vectors; hence IHT and RBR treats these identically. Hyperparameters are tuned via line search and we found $p=s_G$ performs best as $p\ll n$. The resulting ROC curves are plotted in Figures \ref{fig:rochr} and \ref{fig:rocadult}. For both cases, RBR is only slightly worse than XGBoost and noticeably better than IHT. These experiments demonstrate that RBR returns fast and interpretable results (e.g. Figure \ref{fig:nonlinear}) with minimal accuracy cost.

%The fact that RBR performs worse on real dataset compared to synthetic dataset has likely two reasons. First, $p\ll n$, hence there is less need for sparse feature selection. Secondly, synthetic data obeys the additive model \eqref{additive model}.%samet look here (we need to conclude with the advantage of RBR )

\section{Connection to sparse embedding matrices}\label{dim reduct}
An interesting property of our algorithm is the fact that the binned matrix $\Xbin$ can be very sparse as we have $\nnz{\Xbin}=\nnz{\X}$. Let us pause and assume $b=b_i$ for all $1\leq i\leq p$. When Assumption \ref{main assume} holds, $\Xbin$ is a random matrix where each entry is nonzero with probability $b^{-1}$. Surprisingly, similar matrices are being used for dimensionality reduction purposes, namely the Johnson-Lindenstrauss transform where the goal is embedding points to a low-dimensional space while preserving the distances.

In particular, let $\Sb\in\R^{n\times \pb}$ be a matrix where i) Each column contains exactly $s$ nonzeros which are $\pm 1/\sqrt{s}$, ii) Locations and signs of nonzeros are chosen uniformly at random. Then, the linear mapping $\x\rightarrow\Sb\x$ is known to preserve distances with high probability \cite{nelson2013osnap} while allowing $n\ll \pb$. We have the following result that connects random \bdata~to sparse JL transform which also suggests improved sparse dimensionality reduction schemes.
\begin{theorem}\label{distance} Let $\Sb$ be a sparse JL matrix described above with $s$ nonzero entries per column. Let $\Xbin$ be a random \rrb~matrix (i.e.~Assumption \ref{main assume} holds) with bin size $b=n/s$ so that $\nnz{\Xbin}=\nnz{\Sb}$. Flip signs of the nonzero entries of $\Xbin$ at random (each is $\pm 1/\sqrt{s}$ with probability $1/2$). Then, for any unit-length vector $\vb\in \R^{\pb}$, we have that
\[
\E[(\tn{\Sb\vb}^2-1)^2]\geq \E[(\tn{\Xbin\vb}^2-1)^2].
\]
where the inequality is strict if $\nnz{\vb^j}>1$ for some $1\leq j\leq p$.
\end{theorem}
In words, Theorem \ref{distance} shows that the variance of the ``distance deviation'' random variable resulting from $\Xbin$ is lower compared to $\Sb$; hence $\Xbin$ is a good candidate for JL transform for the same number of nonzeros. This result is rather intuitive since $\Xbin$ guarantees orthonormality of columns for each submatrix $\Xbin^j$, unlike $\Sb$.

\section{Conclusions}
In this work, we proposed and analyzed Non-Convex \algo~to learn feature non-linearities via projected gradient descent. The proposed algorithm is computationally efficient, is more expressive than lasso, and has interpretable output.  We empirically demonstrated that the proposed algorithm generally outperforms linear models such as iterative hard thresholding and is competitive with gradient boosted trees (e.g.~XGBoost). As fundamental contributions, we provide statistical and computational guarantees for the RBR algorithm by introducing novel fast-convergence conditions for non-convex projections and by analyzing properties of the \bdata~under random data assumption. %  which is based on the novel restricted gradient correlation condition and properties of the random matrix obtained by feature binning.
We believe our work can be extended in several interesting directions:

$\bullet$ {\bf{Prediction layer of neural networks:}} It would be interesting to apply the feature binning idea at the output layer of deep neural networks. While the backpropagation would be more challenging to implement, our strategy might increase the expressive power of softmax layer and may improve accuracy as well the rate of convergence.

%$\bullet$ {\bf{Equally binned feature matrix:}} Our theoretical results apply to binned matrices with i.i.d. rows. In practice, the borders are data dependent and each bin will have exactly same number of nonzero entries. We need an understanding of this random matrix compared to the current model we analyze.

$\bullet$ {\bf{Sparse embedding matrices:}} In Section \ref{dim reduct}, we demonstrated that the random \bdata~ is a good candidate for the sparse JL transform. This needs further theoretical understanding as well as verification in practical problems. 

%\begin{itemize}
%\item {\bf{As prediction layer of neural networks:}} It would be interesting to apply feature binning idea at the output layer of deep neural networks. While the backpropagation would be more challenging, it might increase the representative power of softmax layer and improve performance in both rate of convergence and accuracy.
%\item {\bf{Equally binned feature matrix:}} Our theoretical results apply to binned matrices with i.i.d. rows. In practice, the borders are data dependent and each bin will have exactly same number of nonzero entries. We need an understanding of this random matrix compared to the current model we analyze.
%\item {\bf{Sparse embedding matrices:}} We demonstrated that random onehot encoded matrices suggests good candidates for sparse JL transform. This needs further theoretical understanding as well as verification in practical problems. 
%\end{itemize}

\newpage

{
%\small
{
\bibliography{Bibfiles}
\bibliographystyle{plain}
}
}
%\newpage

\newpage

\appendix
\section{Projection onto the constraint set}
In this section, we show that RBR Algorithm \ref{pgd-algo} projects $\bbeta$ onto the constraint set \eqref{constraint_eq} (see Theorem \ref{overall proj}). The parameter $\bbeta$ we are estimating has two structures.
\begin{itemize}
\item {\bf{Global:}} $\bbeta$ is group-sparse, because we wish to select a small subset of features. Only $s_G$ out of $p$ $\bbeta^{j}$'s are nonzero.
\item {\bf{Local:}} Feature nonlinearities are represented by subvectors $\bbeta^j$ which are modeled as piecewise constant or piecewise linear. These correspond to sparsity in difference domain $\bbeta^j_{i+1}-\bbeta^j_{i}$, derivative-difference domain or frequency domain.
\end{itemize}

At each projected gradient descent iteration, we wish to perform a projection on this constraint space which we call $\Cc$ as a whole. Let $\Pc_\Cc(\cdot)$ denote the overall projection. We further define the following projections:
\begin{itemize}
\item $\Pc_{G}$: Projects $\bbeta$ onto global group-sparsity constraint.
\item $\Pc_{L}$: Projects individual $\bbeta^j$ onto local-smoothness constraints (i.e.~$s_L$ piecewise-constant or piecewise-linear).
\item $\Pc_{zm}$: Projects individual $\bbeta^j$ onto zero-mean constraints.% denote the projections onto global group-sparsity constraint,and local constraints respectively. 
\end{itemize}
We have the following theorem regarding $\Pc_\Cc$.

\begin{theorem} \label{overall proj}Given a vector $\x\in\R^{\pb}$, the projection on $\Cc$ can be decomposed as
\[
\Pc_\Cc(\x)=\Pc_G(\Pc_L(\Pc_{zm}(\x)))
\]
%projection algorithm is as follows:
%\begin{enumerate}
%\item Obtain local projection $\x_L=\Pc_{L}(\x)=[\Pc_L(\x^1)~\dots~\Pc_L(\x^p)]$.
%\item Calculate distance vector $\d\in\R^p$ where $\d_i=\tn{\x^i}^2-\tn{\x^i-\x_L^i}^2$. Let $\bar{\d}$ be the sorted vector in descending order.
%\item Set $\p^i=\Pc_{L}(\x^i)$ if $\d_i\geq \bar{\d}_k$ and $0$ else.
%\[
%\Pc(\x)=\Pc_{G}(\Pc_{L}(\x))
%\]
%\end{enumerate}
\end{theorem}
\begin{proof} By definition, $\Pc_\Cc(\x)$ is the point closest to $\x$ lying in the  constraint set. By definition $\p=\Pc_\Cc(\x)$ is a vector where subvectors $\p^i$ are zero-mean and satisfy local smoothness and global sparsity. The proof is in two stages. First, let $\Pc_{L,zm}$ be the projection operator to the set of smooth and zero-mean vectors. We show that
\begin{align}
\Pc_\Cc(\x)=\Pc_G(\Pc_{L,zm}(\x))\label{project main 1}
\end{align}
The proof of this result is as follows. Suppose $\p=\Pc_\Cc(\x)$ is the actual projection and let $\p_L=\Pc_{L,zm}(\x)$. This implies that for nonzero blocks $\p^j=\p_L^j=\Pc_{L,zm}(\x^j)$ for all $j\leq p$. Otherwise, replacing $\p^j$ with $\p_L^j$ would result in a strictly shorter distance to $\x$. Consequently, $\p$ is obtained by selecting $s_G$ out of $p$ blocks of $\{\p_L^j\}_{j\leq p}$. Next, observe that we have the identity
\[
\tn{\x^j}^2=\tn{\p^j}^2+\tn{\x^j-\p^j}^2.
\]
This follows from the fact that $\p^j$ is obtained by projecting $\x^j$ onto a closed cone (in our case, cone of zero-mean, piecewise-constant/linear vectors). Consequently
\[
\tn{\x-\p}^2=\sum_{i\leq p} \tn{\x^j-\p^j}^2 = \tn{\x}^2-\tn{\p}^2.
\]
By definition $\p$ minimizes the distance to $\x$ which is same as maximizing $\tn{\p}$. Subject to the constraint $\p$ has $s_G$ nonzero blocks, $\tn{\p}$ is maximized by picking the largest $s_G$ blocks of $\p_L$ which concludes \eqref{project main 1}.

As the next step, we decompose local projection and prove that $\p=\Pc_{L,zm}(\x)=\Pc_L(\Pc_{zm}(\x))$. Proof of this is provided in Lemma \ref{project main 2}.
\end{proof}

%\begin{lemma} Let $\Pc_{PC,k}$ be the $k$ piecewise constant projection. We have that
%\[
%\tn{\x}^2=\tn{\Pc_{PC,k}(\x)}^2+\tn{\x-\Pc_{PC,k}(\x)}^2
%\]
%\end{lemma}
%\begin{proof} Suppose $\p=\Pc_{L}(\x)$ and $\p$ is constant over segments $\{[s_i,e_{i}]\}_{i=1}^{s_L}$ where $e_i=s_{i+1}-1$. Given two endpoints $s_{i},e_{i}$, projection is given by $\text{mean}(\x_{s_{i}:e_{i}})$ and locally satisfies
%\[
%\tn{\p_{s_{i}:e_{i}}}^2+\tn{\x_{s_{i}:e_{i}}-\p_{s_{i}:e_{i}}}^2=\tn{\x_{s_{i}:e_{i}}}^2
%\]
%Summing up over all $1\leq i\leq s_L$, we found the desired result.
%\end{proof}
\begin{lemma}\label{project main 2} Let $\Pc_{L,zm}$ be the projection operator to $s_L$ piecewise-constant/linear vectors with $0$ mean. Then
\[
\Pc_{L,zm}(\x)=\Pc_{L}(\x-\text{mean}(\x))=\Pc_{L}(\Pc_{zm}(\x))%-\text{mean}(\x)
\]
\end{lemma}
\begin{proof} Let $\p=\Pc_{L,zm}(\x)$ and $\p_{zm}=\Pc_{zm}(\x)$. $\p$ is the closest zero-mean point to $\x$ satisfying the local constraints. We need to show that $\p=\Pc_L(\p_{zm})$. % To see this, observe that for any constant $c$
%\[
%\p=\Pc_{L,zm}(\x)=\Pc_{L,zm}(\x-a)=\p_a
%\]
To see this, we first show $\p=\Pc_{L,zm}(\p_{zm})$ as follows. $\tn{\x-\p}^2$ can be decomposed as
\[
\tn{\x-\p}^2=\tn{\x-\mean{\x}\cdot\onebb-\p}^2+\tn{\mean{\x}\cdot\onebb}^2+2\li\mean{\x}\cdot\onebb,\x-\mean{\x}\cdot\onebb-\p\ri.
\]
On the right hand side, the last term is equal to zero, the center term is constant hence $\p$ attempts to minimize the first term which is the distance $\dist{\x-\mean{\x},\p}=\dist{\p_{zm},\p}$. Since projection minimizes the distance, $\p=\Pc_{L,zm}(\p_{zm})$.

Define $\p_L=\Pc_{L}(\p_{zm})$. To conclude, we need to show $\p_L=\p=\Pc_{L,zm}(\p_{zm})$ which is the case if $\mean{\p_L}=0$. Now, observe that
\begin{align}
\tn{\p_{zm}-\p_L}^2=&\tn{\p_{zm}-(\p_L-\mean{\p_L}\cdot\onebb)}^2+\tn{\mean{\p_L}\cdot\onebb}^2\\
&-2\li\mean{\p_L}\cdot\onebb,\p_{zm}-(\p_L-\mean{\p_L}\cdot\onebb)\ri\label{last term last}
\end{align}
On the right hand side, the last term \eqref{last term last} is equal to $0$. The first term satisfies 
\[
\tn{\p_{zm}-(\p_L-\mean{\p_L}\cdot\onebb)}^2\geq \tn{\p_{zm}-\Pc_{L,zm}(\p_{zm})}^2\geq \tn{\p_{zm}-\p_L}^2.
\]
Subtracting $\tn{\p_{zm}-\p_L}^2$ from each side, we find $\tn{\mean{\p_L}\cdot\onebb}^2= 0\implies \mean{\p_L}=0$. This concludes the proof.
%Suppose this is not the case and $\p\neq \p_a$.
%The proof follows from the fact that piecewise constant projection and adding a constant bias are interchangeable operations. In particular
%\[
%\dist{\x-\text{mean}(\x),\Pc_L(\x)}=\dist{\x,\Pc_L(\x)+\text{mean}(\x)}
%\]
\end{proof}

\section{Proof of Theorem \ref{general loss}}
We first provide miscellaneous results on the restricted gradient correlation condition. Let us define restricted smoothness and strong convexity.
\begin{definition} \label{rsc}$\Lc$ satisfies restricted smoothness and strong convexity over set $\Sc$ with parameters $L,\kappa$ if for all $\x,\y\in\Sc$
\[
L \tn{\x-\y}^2\geq \li\x-\y,\grad{\x}-\grad{\y}\ri\geq \kappa \tn{\x-\y}^2.
\]
\end{definition}
\begin{lemma}[RSC implies RGC for linear models] \label{rsc rgc lemma}Suppose $\x\rightarrow \grad{\x}$ is a linear operator. Then, if $\Lc$ satisfies restricted strong smoothness/convexity with $L,\kappa$, over the set $(\Cc-\Cc)+(\Cc-\Cc)$, then it also obeys restricted gradient correlation with $\mur=\frac{2}{L+\kappa}$, $\epsr=3\frac{L-\kappa}{L+\kappa}$.
\end{lemma}%[Proof of Lemma \ref{rsc rgc lemma}]
\begin{proof} We will show that RSC conditions imply Definition \ref{def rgc}. Recall that, for any $\z,\vb\in(\Cc-\Cc)+(\Cc-\Cc)$, we have that
\[
L\tn{\z-\vb}^2\geq \li\vb-\z,\grad{\vb}-\grad{\z}\ri\geq \kappa\tn{\z-\vb}^2
\]
Due to linearity of $\grad{}$, we have $\grad{\vb}-\grad{\z}=\M(\vb-\z)$ for some $\M$.
With this, given $\x,\y,\vb$ we can fix $\z=\x-\y$ and apply RSC on $\z,\vb$. Since RGC is scale invariant, we can assume $\tn{\z}=\tn{\vb}$.
\begin{align}
&L\tn{\z}^2\geq \li\z,\M\z\ri\geq \kappa\tn{\z}^2\\
&L\tn{\vb}^2\geq \li\vb,\M\vb\ri\geq \kappa\tn{\vb}^2\\
&L\tn{\z+\vb}^2\geq \li(\z+\vb),\M(\z+\vb)\ri\geq \kappa\tn{\z+\vb}^2.
\end{align}
Subtracting the first two lines from the last one, we obtain
\[
L\z^T\vb+\frac{L-\kappa}{2}(\tn{\z}^2+\tn{\vb}^2)\geq \li\vb,\M\z\ri \geq \kappa\z^T\vb-\frac{L-\kappa}{2}(\tn{\z}^2+\tn{\vb}^2).
\]
Centering around $0$
\[
\frac{L-\kappa}{2}(\tn{\z}^2+\tn{\vb}^2+\z^T\vb)\geq \li\vb,\M\z\ri-\frac{L+\kappa}{2}\z^T\vb\geq -\frac{L-\kappa}{2}(\tn{\z}^2+\tn{\vb}^2+\z^T\vb).
\]
Taking absolute value and normalizing by $2/(L+\kappa)$
\[
3\tn{\vb}\tn{\x}\frac{L-\kappa}{L+\kappa}\geq \frac{L-\kappa}{L+\kappa}(\z^T\vb+\tn{\z}^2+\tn{\vb}^2)\geq |\frac{2}{L+\kappa}\li\vb,\M\z\ri-\vb^T\z|.
\]
Observing $\li\vb,\M\z\ri=\li\vb,\grad{\z}-\grad{\y}\ri$ and taking the absolute values, we obtain the result.
\end{proof}

\begin{lemma}[RGC implies RSC] Suppose RGC condition of Definition \ref{def rgc} holds. Then, RSC condition outlined in Definition \ref{rsc} holds with set $\Cc-\Cc$ and parameters $L=(1+\epsr)/\mur,~\kappa=(1-\epsr)/\mur$.
\end{lemma}
\begin{proof} Given $\vb,\x,\y$ described in Definition \ref{def rgc}, setting $\vb=\x-\y$, we have
\[
|\tn{\x-\y}^2-\mur\li\x-\y,\grad{\x}-\grad{\y}\ri|\leq \epsr\tn{\x-\y}^2.
\]
This implies $\frac{1+\epsr}{\mur}\tn{\x-\y}^2\geq \li\x-\y,\grad{\x}-\grad{\y}\ri\geq \frac{1-\epsr}{\mur}\tn{\x-\y}^2$.
\end{proof}

The following is a standard result on the properties of subgaussian matrices. This result is useful for the proof of Theorem \ref{main binning} to show that binned matrix satisfies RSC/RGC conditions.
\begin{proposition} [Subgaussian RSC \cite{mendelson2007reconstruction, companion}] \label{propo copy}Suppose $\X\in\R^{n\times \pb}$ is a matrix with independent subgaussian rows with bounded subgaussian norm (by a constant) and each row has identity covariance. Given a cone $\Cc\in\R^{\pb}$ and $\eps>0$, there exists constants $c=c, C$ so that if $n>c\eps^{-2}\omega(\Cc)$ for all unit length $\vb,\w\in \Cc$, with probability $1-\exp(-Cn)$, $\X$ obeys
\[
|n^{-1}\vb^T\X^T\X\w-\vb^T\w|\leq \eps.
\]
% Suppose $\bbeta\in \R^p$ is the hidden parameter and data is generated as $\bb=\X\bbeta+\z\in\R^n$ and the loss function is quadratic i.e. $\Lc(\bbeta)=\tn{\bb-\X\bbeta}^2$. Suppose $\sqrt{n}\X$ is a matrix with independent and isotropic subgaussian rows. Then, given $\eps>0$, there exists constants $c,C$ such that if $n\geq c\omega(\Cc-\Cc)$, restricted gradient correlation holds with $\mu=1$ and $\eps=\eps$ with probability $1-c\exp(-Cn)$.
\end{proposition}

\subsection{Proof of Theorem \ref{general loss}}
%Compared to \cite{companion}, we have two novelt
\begin{proof} Our proof strategy borrows ideas from \cite{companion}. We first show that $\bbeta_t-\bbetas$ converges to zero. Observe that $\Cc-\{\vb\}$ is the set of feasible directions at $\vb$ namely $\{\ub\big|\ub+\vb\in\Cc\}$. Let $\cone{S}=\text{Cl}(\{\alpha \s~\big|~\alpha\geq 0,~\s\in S\})$ i.e. the closure of cone of $S$. Define tangent cone to be $\Tc_{\Cc,\vb}=\cone{\Cc-\{\vb\}}$. PGD iterations for $\bbeta_{t+1}$ obey
\begin{align}
\tn{\bbeta_{t+1}-\bbetas}&=\tn{\Pc_{\Cc}(\bbeta_{t}-\mur\grad{\bbeta_t})-\bbetas}\\
&= \tn{\Pc_{\Cc-\{\bbetas\}}(\bbeta_{t}-\bbetas-\mur\grad{\bbeta_t})}\\
&\leq 2\tn{\Pc_{\Tc_{\Cc,\bbetas}}(\bbeta_{t}-\bbetas-\mur\grad{\bbeta_t})}\label{two factor}\\
&= 2\sup_{\w\in\Tc_{\Cc,\bbetas},\tn{\w}\leq1}\li\w,\bbeta_{t}-\bbetas-\mur\grad{\bbeta_t}\ri\\
&= 2\sup_{\w\in\Tc_{\Cc,\bbetas},\tn{\w}\leq1}\li\w,\bbeta_{t}-\bbetas-\mur(\grad{\bbeta_t}-\grad{\bbetas})-\mur\grad{\bbetas}\ri\\
&\leq 2\sup_{\w\in\Tc_{\Cc,\bbetas},\tn{\w}\leq 1}\li\w,\bbeta_{t}-\bbetas-\mur(\grad{\bbeta_t}-\grad{\bbetas})\ri\label{one before last}\\
&\hspace{5pt}+2\mur\sup_{\w\in\Tc_{\Cc,\bbetas},\tn{\w}\leq1}\li\w,-\grad{\bbetas}\ri\label{last term}.
\end{align}
\eqref{two factor} follows from Lemma $6.4$ of \cite{companion}. In the final line, optimality conditions (KKT) imply that 
\[
\inf_{\w\in\Tc_{\Cc,\bbetas},\tn{\w}\leq 1}\li\w,\grad{\bbetas}\ri\geq 0
\] hence \eqref{last term} is non-positive. Using the fact that $\w\in \Tc_{\Cc,\bbetas}\subset\Cc-\Cc$, the restricted gradient correlation bounds the first term \eqref{one before last} as
\[
\sup_{\w\in\Tc_{\Cc,\bbetas},\tn{\w}\leq 1}\li\w,\bbeta_{t}-\bbetas-\mur(\grad{\bbeta_t}-\grad{\bbetas})\ri\leq \epsr\tn{\bbeta_t-\bbetas}.
\]
Combining, we obtain $\tn{\bbeta_{t+1}-\bbetas}\leq 2\epsr \tn{\bbeta_{t}-\bbetas}$ which gives the linear convergence
\[
\tn{\bbeta_{t}-\bbetas}\leq (2\epsr)^t \tn{\bbetas}
\]
%Finally, use the trivial bound 
%\[
%\Lc(\bbeta_t)-\Lc(\bbetas)\leq \li\bbeta_t-\bbetas,\grad{\bbetas}\ri+
%\]
%\[
%\kappa\tn{\x-\y}^2\li\x-\y,\grad{\x}-\grad{\y}\ri\leq L\tn{\x-\y}^2
%\]
%\[
%\li\x-\y,\grad{\x}\ri\geq \Lc(\x)-\Lc(\y)\geq \li\x-\y,\grad{\y}\ri
%\]
%\[
%\li\bbeta_t-\bbetas,\grad{\bbeta_t}\ri\geq \Lc(\bbeta_t)-\Lc(\bbetas)
%\]
To achieve the convergence of loss function, observe that
\begin{align}
\frac{\epsr+1}{\mur}\tn{\bbeta_t-\bbetas}^2+&\li\bbeta_t-\bbetas,\grad{\bbetas}\ri\\
&\geq \li\bbeta_t-\bbetas,\grad{\bbeta_t}-\grad{\bbetas}\ri+\li\bbeta_t-\bbetas,\grad{\bbetas}\ri\geq \Lc(\bbeta_t)-\Lc(\bbetas).\nn
\end{align}
On the left hand side, we upper bound $\li\bbeta_t-\bbetas,\grad{\bbetas}\ri\leq \tn{\bbeta_t-\bbetas}\tn{\grad{\bbetas}}$ to conclude.

The second statement follows an identical argument.
\begin{align}
\tn{\bbeta_{t+1}-\btrue}&=\tn{\Pc_{\Cc}(\bbeta_{t}-\mur\grad{\bbeta_t})-\btrue}\\
&= \tn{\Pc_{\Cc-\{\btrue\}}(\bbeta_{t}-\btrue-\mur\grad{\bbeta_t})}\\
&\leq 2\tn{\Pc_{\Tc_{\Cc,\btrue}}(\bbeta_{t}-\btrue-\mur\grad{\bbeta_t})}\\
&= 2\sup_{\w\in\Tc_{\Cc,\btrue},\tn{\w}\leq1}\li\w,\bbeta_{t}-\btrue-\mur\grad{\bbeta_t}\ri\\
&= 2\sup_{\w\in\Tc_{\Cc,\btrue},\tn{\w}\leq1}\li\w,\bbeta_{t}-\btrue-\mur(\grad{\bbeta_t}-\grad{\btrue})-\mur\grad{\btrue}\ri\nn\\
&\leq 2\sup_{\w\in\Tc_{\Cc,\btrue},\tn{\w}\leq 1}\li\w,\bbeta_{t}-\btrue-\mur(\grad{\bbeta_t}-\grad{\btrue})\ri\\
&~~~+2\mur\sup_{\w\in\Tc_{\Cc,\btrue},\tn{\w}\leq 1}\li\w,-\grad{\btrue}\ri\\
&\leq 2\epsr \tn{\bbeta_t-\btrue}+2\mur\tn{\Pc_{\Cc-\Cc}(-\grad{\btrue})}.
\end{align}
The recursion in the last line implies
\[
\tn{\bbeta_{t}-\btrue}\leq (2\epsr)^t\tn{\btrue}+\frac{2\mur}{1-2\epsr}\tn{\Pc_{\Cc-\Cc}(-\grad{\btrue})}
\]
which is the advertised bound.
\end{proof}

\section{Proof of Theorem \ref{main binning}}\begin{proof}
The proof follows by combining the main results of Sections \ref{sec:samp comp} and \ref{sec:rand bin}. The main ingredient is Theorem \ref{general loss}. First, we show \irb~matrix obeys restricted gradient correlation with $n>n_0$ samples. To show this, first observe that Proposition \ref{propo copy} is applicable since Theorems \ref{subgauss rows} and \ref{good one} proves that $\Xbin$ is a subgaussian matrix. Then, we combine Theorem \eqref{samp comp} and Proposition \ref{propo copy} to deduce that under given conditions (i.e. $n>n_0$) for all $\tn{\x}=1,\x\in\Cc-\Cc$, 
\[
1+\eps\geq \tn{\Xbin\x}\geq 1-\eps
\]
as long as $n>c\eps^2\omega(\Cc-\Cc)$. Observe that $\Cc-\Cc$ is a subset of the $2s_G$ group sparse vectors which are $2s_L$ locally smooth hence $\omega(\Cc-\Cc)\leq n_0=cs_G(s_L\log  \max_jb_j+\log p)\leq c s_Gs_L\log \pb$.

Next, note that $\grad{\x}-\grad{\y}=\Xbin^T\Xbin(\x-\y)$ which means
\[
(1+\eps)^2\geq\li\x-\y,\grad{\x}-\grad{\y}\ri=\tn{\Xbin(\x-\y)}^2\geq (1-\eps)^2.
\]
Now, using Lemma \ref{rsc rgc lemma}, we immediately find that RGC condition is satisfied with $\mur=1$ and $\epsr$ grows as $\sqrt{n_0/n}$. This yields the linear convergence bound. To obtain the statistical precision term, we note that $-\grad{\btrue}=\Xbin^T\z$ hence applying Theorem \ref{general loss} the error term is
\[
\Pc_{\Cc-\Cc}(\Xbin^T\z)=\sup_{\tn{\vb}\leq 1,\vb\in\Cc-\Cc}\vb^T(\Xbin^T\z).
\]
$\frac{\Xbin^T\z}{\tn{\z}}$ is a subgaussian vector hence classical generic chaining bounds \cite{talagrand2006generic} imply that ``subgaussian width'' is upper bounded by a constant of Gaussian width i.e.
\[
\sup_{\tn{\vb}\leq 1,\vb\in\Cc-\Cc}\vb^T(\Xbin^T\z)\leq c\omega(\Cc-\Cc)^{1/2}\tn{\z}.
\]
These finishes the proof of the first statement.

The second statement is based on the idea of approximating \rrb~with \irb. This is outlined in Theorem \ref{irb to rrb}. In particular, as long as 
\[
s_G\leq c\sqrt{n/(\max_jb_j\log\pb)}
\]
for sufficiently small $c>0$, the restricted gradient correlation holds with sufficiently small $\epsr$ with $\mur=1$ hence, we have linear convergence.

For the error bound, we use the weaker estimate
\[
\Pc_{\Cc-\Cc}(\Xbin^T\z)\leq \sup_{\tn{\vb}\leq 1,\vb\in \Cc-\Cc}\vb^T\Xbin^T\z\leq \tn{\z}\sup_{\tn{\vb}\leq 1,\vb\in \Cc-\Cc}\tn{\Xbin\vb}\leq \tn{\z}(1+\eps)
\]
due to lack of subgaussianity.
\end{proof}

\section{Proof of Theorem \ref{main binning} - Sample complexity analysis}\label{sec:samp comp}

As discussed in the proof sketch of Theorem \ref{main binning}, a crucial estimate is the Gaussian complexity of the constraint set $\Cc$. First, we provide a rigorous definition of Gaussian complexity.
\begin{definition} [Gaussian complexity] Given a closed cone $\Cc\in\R^d$ and a standard normal vector $\g\in\R^d$, Gaussian complexity of $\Cc$ is defined as
\[
\omega(\Cc)=\E[\tn{\Pc_\Cc(\g)}^2]=\E[\sup_{\vb\in\Cc,\tn{\vb}\leq 1}(\vb^T\g)^2]
\]
\end{definition}
Given a finite set of scalars $\{a_i\}_{i\geq 1}$, let $\text{sumtop}_k(\{a_i\}_{i\geq 1})$ return the sum of top $k$ elements. The theorem below, summarizes our result on Gaussian complexity of the constraint set $\Cc$ described in \eqref{constraint_eq}.%that arises from Algorithm \ref{pgd-algo}.
\begin{theorem}\label{samp comp} Relax the definition of $\Cc$ in \eqref{constraint_eq} by allowing distinct local smoothness values $\{s_{L,j}\}_{j\leq p}$. Then, for some absolute constant $c>0$, $\omega(\Cc)$ obeys
\[
\omega(\Cc)\leq c \times \text{sumtop}_{s_G}(\{s_{L,j}\log b_j+\log p\}_{1\leq j\leq p}).
\]
\end{theorem}
\begin{proof} The proof is achieved by combining Theorem \ref{theorem-sparse} and Theorem \ref{local-sample}. Denoting the local constraint sets by $\Sc_j$, Theorem \ref{theorem-sparse} shows
\begin{align}
\omega(\Cc)&\leq (\sqrt{\text{sumtop}_{s_G}\{\omega(\Sc_j)\}_{j\leq p}}+c\sqrt{s_G\log p})^2\leq 2\text{sumtop}_{s_G}\{\omega(\Sc_j)\}_{j\leq p}+2cs_G\log p.\\
&\leq 2\text{sumtop}_{s_G}\{\omega(\Sc_j)+c\log p\}_{j\leq p}
\end{align}
Next, Theorem \ref{local-sample} bounds $\omega(\Sc_j)\leq c_2s_{L,j}\log b_j$ which results in the desired bound.
\end{proof}

\begin{theorem}[Sparse feature selection with local smoothness] \label{theorem-sparse}Let $\g\in\Nn(0,\Iden_{\pb})$. Let $\Cc\subset\R^{\pb}$ be the set of $s_G$ group sparse vectors $\vb$ where each subvector $\vb^j$ lies in a cone $\Sc_j$ which stands for the smoothness constraint. For some constant $c>0$, we have that%satisfies $\Lc$ property. We have that
\[
\omega(\Cc)\leq  (\sqrt{\text{sumtop}_{s_G}\{\omega(\Sc_j)\}_{j\leq p}}+c\sqrt{s_G\log p})^2.
\]
\end{theorem}
\begin{proof} Suppose $\vb\in\Cc$ obeying $\tn{\vb}\leq 1$ and $\g$ is a standard normal vector. We have that
\begin{align}
\vb^T\g&=\sum_{j\leq p}\li\vb^j,\g^j\ri\leq \sum_{j\leq p,\vb^j\neq 0}\tn{\Pc_{\Sc_j}(\g_j)}\tn{\vb^j}\\
&\leq \sqrt{\sum_{j\leq p,\vb^j\neq 0}\tn{\Pc_{\Sc_j}(\g_j)}^2}\sqrt{\sum_{j}\tn{\vb^j}^2}\\
&\leq \sqrt{\text{sumtop}_{s_G}(\tn{\Pc_{\Sc_j}(\g_j)}^2)}\label{sumtop}
\end{align} 
where the second line follows from Cauchy-Schwarz and the last line follows from the fact that $\tn{\vb}\leq 1$.

Using the fact that projection is a $1$-Lipschitz function (Lemma \ref{lipconst}), and applying Gaussian concentration of Lipschitz functions, for any $\Sc_j$, we have that
\[
\Pr(\tn{\Pc_{\Sc_j}(\g_j)}\geq \sqrt{\omega(\Sc_j)}+2\tau)\leq \exp(-\tau^2).
\]
Applying a union bound over $p$ instances and setting $\tau\rightarrow \tau+\sqrt{\log p}$ we obtain that
\[
\Pr(\max_{j\leq p}(\tn{\Pc_{\Sc_j}(\g_j)}-\sqrt{\omega(\Sc_j)})\geq \tau)\leq \exp(-\tau^2).
\]
Define $q=\text{sumtop}_{s_G}(\{\omega(\Sc_j)\}_{j\leq p})$ and $l=\text{sumtop}_{s_G}(\{\sqrt{\omega(\Sc_j)}\}_{j\leq p})\leq \sqrt{s_Gq}$. Substituting into \eqref{sumtop}, with the same $\exp(-\tau^2)$ probability
\begin{align}
\sup_{\vb\in\Cc}(\vb^T\g)^2&\leq \text{sumtop}_{s_G}(\{(\sqrt{\omega(\Sc_j)}+2(\tau+\sqrt{\log p}))^2\}_{j\leq p})\\
&\leq q+4l(\tau+\sqrt{\log p})+s_G4(\tau+\sqrt{\log p})^2\\
&\leq q+4\sqrt{s_Gq}(\tau+\sqrt{\log p})+s_G4(\tau+\sqrt{\log p})^2\\
&\leq (\sqrt{q}+2\sqrt{s_G}(\sqrt{\log p}+\tau))^2
\end{align}
This implies that the left hand-side is a random variable with a exponential tail bound. The expectation of such a random variable can be upper bounded by a standard integration-by-parts trick to yield
\[
\E[\sup_{\vb\in\Cc}(\vb^T\g)^2]\leq \sqrt{q}+c\sqrt{s_G\log p}
\]
for some constant $c>0$ (assuming $p>2$).
\end{proof}

\begin{theorem} \label{local-sample}Let $\Sc\subset\R^{b}$ be the set of $s_L$ piecewise-constant (or piecewise-linear functions). For some constant $c>0$, we have that
\[
\omega(\Sc)\leq cs_L\log b.
\]
\end{theorem}
%$\E[\sup_{\x\in \Sc,\tn{\x}=1} \li\x,\g\ri^2]\leq$
\begin{proof} We will prove the result for piecewise-constant functions. Piecewise-linear proof follows similarly as discussed below. First we show the following result.
\begin{lemma} \label{pc lemma}Given a standard normal vector $\g\in\R^b$, with probability $1-\exp(-\tau^2/2)$ the following holds. Consider the set of segments $\Ic$ starting at entry $i_s$ and ending at entry $i_e$. We have that
\begin{align}
%|\sum_{i\in \Ic} g_i|\leq \sqrt{|\Ic|}(\tau+\sqrt{\log b})
\E[\sup_{s_i,e_i}\frac{\onebb^T\g_\Ic}{\sqrt{|\Ic|}}]\leq c\sqrt{\log b}\label{main local}
\end{align}
\end{lemma}
\begin{proof} For a given segment $\Ic$, $\onebb^T\g_\Ic$ has distribution $\Nn(0,|\Ic|)$. Hence, it satisfies
\[
\Pro(|\onebb^T\g_\Ic|\geq \sqrt{|\Ic|}\tau)\leq \exp(-\tau^2/2).
\]
There are $b^2$ different start/end point pairs. By using $\tau\rightarrow\tau+c\sqrt{\log b}$ and applying a union bound, we obtain that for all segments, $|\onebb^T\g_I|\geq \sqrt{|\Ic|}c(\tau+\sqrt{\log b})$ holds. Alternately, we can write
\[
\sup_{s_i,e_i}\frac{\onebb^T\g_\Ic}{\sqrt{|\Ic|}}\leq c(\tau+\sqrt{\log b}).
\]
Applying a standard integration-by-parts obtains the result \eqref{main local}.
\end{proof}
Define $f(\g)=\sup_{s_i,e_i}\frac{\onebb^T\g_\Ic}{\sqrt{|\Ic|}}$. Now, suppose $\x\in\Sc$ is an $s_L$ piecewise constant vector with unit $\ell_2$ norm. Suppose $\x$ has segments $\{[s_i,e_i]\}_{i=1}^{s_L}$. Also denote $\Ic_i=[s_i,e_i]$ so that $\x_{\Ic_i}=\mean{\x_{\Ic_i}}\cdot\onebb$. We can write
\[
\x^T\g=\sum_{i\leq s_L} \x_{\Ic_i}^T\g_{\Ic_i}=(\mean{\x_{\Ic_i}}\cdot\onebb)^T\g_{\Ic_i}\leq \sum_{i\leq s_L}|\mean{\x_{\Ic_i}}|\sqrt{|\Ic|}f(\g).
\]
Now, observe that
\[
\sum_{j\leq s_L} |\mean{\x_{\Ic_i}}|\sqrt{|\Ic|}f(\g)=\sum_{j\leq s_L}\tn{\x_{\Ic_i}} f(\g)\leq \sqrt{s_L}f(\g).
\]
Hence $(\x^T\g)^2\leq s_Lf(\g)^2\leq c^2 s_L\log b$ where we applied \eqref{main local}.

Piecewise linear functions are represented as $ax+b$ where $b$ is the constant term's and $a$ is the linear term's coefficient. For the proof, we simply obtain a variation of Lemma \ref{pc lemma} where the constant term vector $\onebb$ is replaced by the linear term vector with entries $\vb_{\text{lin}}=[-(b-1)/2,-(b-3)/2,\dots,(b-1)/2]$. $\vb_{\text{lin}}$ is orthogonal to the $\onebb$ vector and together they characterize the linear approximation of $\g_\Ic$.
\end{proof}
\begin{lemma}[Lipschitz constant of supremum]\label{lipconst} Suppose $S$ is a subset of the unit $\ell_2$ ball. The function $f(\g)=\sup_{\vb\in S}\g^T\vb$ is a $1$-Lipschitz function of $\g$.
%Suppose $\x\in\R^b$ is a $s_L$ piecewise-constant vector and $\g\sim\Nn(0,\Iden_b)$. Then
%\[
%f(\g)=\sup_{\x\in \PC_k,\tn{\x}=1}\x^T\g
%\]
%is an $\sqrt{k}$-Lipschitz function.
\end{lemma}
\begin{proof} Given $\g_1,\g_2$ let us investigate $f(\g_1)-f(\g_2)$. Let $\x_1,\x_2\in S$ be the maximizing vectors. We have that
\[
f(\g_1)-f(\g_2)\leq \x_1^T\g_1-\x_2^T\g_2\leq \x_1^T(\g_1-\g_2)\leq \tn{\g_1-\g_2}\tn{\x_1}.
\]
Observe $\tn{\x_1}\leq 1$ to conclude.
\end{proof}

\section{Proof of Theorem \ref{main binning} - Properties of random binned data matrices}\label{sec:rand bin}
In this section we prove useful properties of the \bdata~matrix when $\X$ obeys Assumption \ref{main assume} and when binning strategies follow Definition \ref{strategy}.
\subsection{Analysis of random encoding vector}

\begin{definition} [Random encoding vector] \label{rev}$\ab\in\R^{\pb}$ is called a random encoding vector if $\{\ab^j\}_{j=1}^p$ are $b_j$ dimensional independent vectors and $\ab^j/\sqrt{b_j}$ is uniformly distributed over the standard basis.
\end{definition}
\begin{theorem}\label{subgauss rows} Let $\ab$ be a random feature encoding as described in Definition \ref{rev}. Then, over the subspace $S_{zm}=\{\vb\in\R^{\pb}~\big|~\onebb^T\vb^j=0,~1\leq j\leq p\}$, $\ab$ is a zero-mean vector with identity covariance and subgaussian norm bounded by $\order{\max_{1\leq j\leq p}\sqrt{b_j}}$.
\end{theorem}
\begin{proof} First, we argue that each subvector $\ab^j$ is a zero-mean subgaussian vector for each $j$ over the space $\onebb^T\vb^j=0$. Observe that $\E[\ab^j]=b_j^{-1}\onebb$. This implies that $\E[\ab]$ is in the span of $\onebb$ and $\ab-\Pc_{\onebb}(\ab)$ is zero-mean. Next, we obtain the correlation matrix of $\ab$ as
\[
\E[\ab^j{\ab^j}^T]-\E[\ab^j]\E[\ab^j]^T = b_j^{-1}\Iden-b_j^{-2}\onebb\onebb^T
\]
which is the identity matrix scaled by $b_j^{-1}$ over the complementary space of $\onebb$. Hence, $\ab^j-\Pc_{\onebb}(\ab^j)$ has identity covariance over the  space $\{\vb^j~\big|~\onebb^T\vb^j=0\}$. Next, given unit length $\x$, we study the subgaussianity of $\li\x,\ab^j\ri$ as follows
\[
\E[|\li\x,\ab^j\ri|^n]=b_j^{-1}\E[|\sum_i x_i^n|]\leq b_j^{-1}.
\]
Following the definitions in \cite{Vers}, this implies that the subgaussian norm of $\x^T\ab^j$ is bounded by a constant which implies that subgaussian norm of $\ab^j$ is at most $\order{\sqrt{b_j}}$ after scaling. Next, for the concatenated vector, we have that
\[
%\E[|\x^T\ab|^n]=\E[|\sum_{i\leq p}\x_i^T\ab_i|^n]
\x^T\ab=\sum_{1\leq j\leq p}{\x^j}^T\ab^j.
\]
This implies that $\x^T\ab$ is sum of independent random variables and the $j$th random variable ${\x^j}^T\ab^j$ is zero-mean subgaussian with subgaussian norm proportional to $\sqrt{b_j}\tn{\x^i}$. Applying results of \cite{Vers}, this implies that for unit length $\x$, $\x^T\ab$ is subgaussian with norm proportional to $\sqrt{\sum_j b_j\tn{\x^i}^2}\leq \max_{j}\sqrt{b_j}$. This is the very definition of subgaussian norm of a random vector.
\end{proof}

\subsection{\irb~matrices~have independent rows}
%It is crucial that when $\X$
\begin{theorem} \label{good one}Suppose $\X$ obeys Assumption \ref{main assume} and $\Xbin$ is generated by binning $\X$ according to \irb~strategy. Then, $\Xbin$ has independent columns each of which are random encoding vectors scaled by $\sqrt{n}$.
\end{theorem}
\begin{proof} Suppose $\Xbin$ is generated according to \irb~scheme. First observe that feature columns $\Xbin^j$ are independent of each other. The reason is that $\X$ has independent columns and different features are binned independently. What remains to show is that $\Xbin^j$ has independent rows. To show this, we will argue that the nonzero pattern of $\Xbin^j$ is statistically identical to a matrix $\B$ with independent rows where each row is standard basis element chosen uniformly at random.

Let us understand the properties of $\B$. Each entry of $\B$ is a Bernoulli $b_j^{-1}$ trial, in particular independent rows imply that number of nonzeros in each column is a Binomial $(n,b_j^{-1})$ variable. Viewing each column sum as a random variable, the only dependence is that the random variables add up to $n$. By construction (i.e.~Definition \ref{strategy}), number of nonzeros per column of $\Xbin^j$ is statistically identical to that of $\B$. We next argue that nonzero support of each column is chosen uniformly at random to conclude. This follows from the fact that $\Xbin^j$ has i.i.d. columns. Since rows of $\X$ are i.i.d., the samples that lie on a particular quantile (of size $\nnz{\Xbinn{:,i}^j}$) of the empirical distribution is chosen uniformly at random hence the nonzero support of that column is uniformly random.
\end{proof}
%\begin{lemma} Assume $\rb\in\R^B$ is a random one-hot encoded $\pm\sqrt{B}$ vector. Given unit length $\x$, we have that
%
%\end{lemma}
%\begin{proof} Observe that \[
%\E[|\rb^T\x|^k]=\frac{\sum_{i\leq B} B^{k/2}|\x_i|^k}{B}
%\]
%For $k=2$, this implies that $\E[|\rb^T\x|^k]=1$. For $k>2$, we have that $\E[|\rb^T\x|^k]\leq B^{k/2-1}\tn{\x}=B^{k/2-1}$. Finally, for $k=1$, $B^{-1/2}\sum_{i}|\x_i|\leq 1$. These imply that $\E[|\rb^T\x|^k]^{1/k}\leq \sqrt{B}$ hence $\rb$ is a zero-mean subgaussian random vector with subgaussian norm at most $c\sqrt{B}$.
%\end{proof}
%
%\begin{lemma} Suppose $\rb$ is a subgaussian vector with norm $B$ and suppose $\{\rb_i\}_{i=1}^p$ are i.i.d. replicas of $\rb$. Let $\rb_{all}=[\rb_1~\dots~\rb_p]$. Then, $\rb_{all}$ is a subgaussian vector with same norm $B$.
%\end{lemma}
%\begin{proof} For any $\x$, we have that $\rb_{all}^T\x=\sum_{i\leq p} \rb_i^T\x_i$ where $\rb_i^T\x_i$ has subgaussian norm $\sqrt{B}\tn{\x_i}$ and $\rb_i^T\x_i$ are independent of each other. Then, using Lemma $5.9$ of \cite{Vers}, subgaussian norm of $\rb_{all}^T\x$ is at most $c(\sum_{1\leq i\leq p} B\tn{\x_i}^2)^{1/2}=c\sqrt{B}$.
%\end{proof}
\subsection{Analysis of \rrb}
In this section, we provide an analysis of the \rrb~strategy. First, in the light of Theorem \ref{good one}, we describe the statistical properties of \rrb~matrix. Using the same argument, assuming Assumption \ref{main assume} holds, a \rrb~matrix satisfies
\begin{itemize}
\item Nonzero pattern of each column is chosen uniformly at random.
\item Columns in the same feature submatrix $\Xbin^j$ have non-overlapping patterns.
\item Columns from different feature groups ($\Xbin^j,\Xbin^k$ where $j\neq k$) are independent.
\end{itemize}

With this observations, the idea of the proof is approximating the \irb~matrix with a \rrb~matrix. We first have the following result on restricted smoothness / strong convexity of constraint set $\Cc$ defined in \eqref{constraint_eq}.
\begin{theorem} Let $\Cc$ be the set defined at \eqref{constraint_eq}. Let $\Gc$ be the set of $s_G$ group sparse vectors and $\A$ and $\B$ be two matrices of size $\R^{n\times \pb}$. Let $\Sc^{n-1}$ be the unit $\ell_2$ sphere. We have the relations
\begin{align}
&\min_{\vb\in \Cc\cap \Sc^{n-1}}\tn{\A\vb}\geq \min_{\vb\in \Cc\cap \Sc^{n-1}}\tn{\B\vb}-\max_{\vb\in\Gc\cap \Sc^{n-1}}\tn{(\A-\B)\vb},\label{first diff bound}\\
&\max_{\vb\in \Cc\cap \Sc^{n-1}}\tn{\A\vb}\geq \max_{\vb\in \Cc\cap \Sc^{n-1}}\tn{\B\vb}-\max_{\vb\in\Gc\cap \Sc^{n-1}}\tn{(\A-\B)\vb}.
\end{align}
\end{theorem}
\begin{proof}We prove \eqref{first diff bound}. Proof of the second statement is identical. The main idea is the fact that $\Cc\subset\Gc$. Consequently, for any $\vb\in\Cc$,
\[
\min_{\vb\in \Cc\cap \Sc^{n-1}}\tn{\A\vb}\geq \min_{\vb\in \Cc\cap \Sc^{n-1}}(\tn{\B\vb}-\tn{(\A-\B)\vb})\geq \min_{\vb\in \Cc\cap \Sc^{n-1}}\tn{\B\vb}-\max_{\vb\in\Gc\cap \Sc^{n-1}}\tn{(\A-\B)\vb}.
\]
\end{proof}
Our next theorem approximates the spectral norms (denoted $\|\cdot\|$) of submatrices of \rrb~in terms of that of \irb~and shows that \rrb~matrices enjoys good restricted gradient correlation bounds over set $\Cc$
\begin{theorem}\label{irb to rrb} Let $\Gc$ be the set of $s_G$ group sparse vectors. Let $\Gbin$ be a \irb~matrix. Then, with probability $1-\pb^{-10}$, there exists $\Bbin$ such that $\Bbin$ is statistically identical to a \rrb~matrix and satisfies
\[
\max_{\vb\in\Gc\cap \Sc^{n-1}}\tn{(\Bbin-\Gbin)\vb}\leq c(s_G^2\log \pb \max_jb_j/n)^{1/4}.
\]
%that uniform binning strategy where each column contains $n/b$ nonzero entries. 
\end{theorem}
\begin{proof} Given $\Gbin$, we construct $\Bbin$ by constructing $\Bbin^j$ from $\Gbin^j$ for all $j\leq p$. Fix the column block $j$. Create sets $S_+$ and $S_-$ which are the sets of columns of $\Gbin^j$ with number of nonzeros more than and less than $n/b_j$ respectively. Let $\supp{}$ return the nonzero support of a vector/matrix. For a column $c\in S_+$, set $\supp{\Bbinn{:,c}^j}$ to be uniformly random subset of the $\supp{\Gbin^j_{:,c}}$ of size $n/b_j$. Let $E_+$ be the excess nonzero coordinates that is given by 
\[
E_+=\bigcup_{c\in S_+}\supp{\Gbin^j_{:,c}}-\supp{\Bbinn{:,c}^j}.
\] 

Next, we construct $\Bbin$ over $S_-$ as follows. For any $c\in S_-$, $c$th column obeys $\supp{\Bbinn{:,c}^j}\supset \supp{\Gbin^j_{:,c}}$. Observe that $\supp{\Bbinn{:,c}^j}$ is missing $n/b_j-\supp{\Gbin^j_{:,c}}$ coordinates. We pick these coordinates uniformly at random from the set $E_+$ for each $c\in S_-$. It can be shown that this construction of $\Bbin$ results in \rrb~matrix because it is ignorant of the permutation of the nonzero assignments.

With this construction, we will proceed with the analysis of the spectral norm $\|\cdot\|$. First defining $\Db=\Bbin-\Gbin$, observe that supports of columns of $\Db^j=\Bbin^j-\Gbin^j$ over $S_+$ and over $S_-$ are nonoverlapping which means individually $\Db^j_{:,S_+}$ and $\Db^j_{:,S-}$ have orthogonal columns. Hence, the spectral norm of each one is simply the column with the largest $\ell_2$ norm. Two estimate this, we shall bound the number of nonzeros of each column. This can be done by using a Chernoff bound as follows. $\nnz{\Gbin^j_{:,c}}$ is sum of $n$ independent Bernoulli's with mean $b_j^{-1}$. Hence the mean is $n/b_j$. Subtracting the overlap with $\Bbinn{:,c}^j$ corresponds to studying the deviation of the sum from its mean as $\nnz{\Bbinn{:,c}^j}=n/b_j$. Consequently, applying Chernoff bound, for any $j,c$ pair,
\[
\Pr(\nnz{\Db^j_{:,c}}\geq \eps n/b_j)\leq \exp(-\eps^2b_j^{-1}n/4).
\]
Setting $\eps=C\sqrt{\log \pb b_j/n}$ for some constant $C>0$, and union bounding over all $\pb$ columns of $\Db$, we obtain that with probability $1-\pb^{-10}$, $\nnz{\Db^j_{:,c}}\leq C\sqrt{n/b_j\log \pb}$ is satisfied for all $j,c$.
%Setting $\eps=C\sqrt{\log \pb b_j/n}$ for some constant $C>0$ and applying this bound for all $\pb$ columns $1\leq c\leq\pb$ of $\Db$, union bounding, we obtain that with probability $1-\pb^{-10}$, $\nnz{\Db^j_{:,c}}\leq C\sqrt{n/b_j\log \pb}$ is satisfied.

This implies that, with column length normalization $\sqrt{n/b_j}$, we obtain that for all columns $c$ obey,
\[
\tn{\Db^j_{:,c}}=\sqrt{b_j/n}\sqrt{\nnz{\Db^j_{:,c}}}\leq C(\log \pb b_j/n)^{1/4}.
\]
Consequently $\|\Db^j\|\leq \|\Db^j_{:,S-}\|+\|\Db^j_{:,S-}\|\leq 2C(\log \pb b_j/n)^{1/4}$. This provides spectral norm control of individual submatrices $\Db^j$. We are interested in arbitrary $s_G$ size concatenation of them. From Lemma \ref{concated} spectral norm of the concatenation of $s_G$ matrices scales as $\sqrt{s_G}$.

Using the spectral norm upper bound $\sqrt{s_G}2C(\log \pb b_j/n)^{1/4}$ on the $s_G$-size block submatrices of $\Db$, for any $s_G$ group sparse vector unit length $\x$
\[
\tn{\Db\x}\leq 2C\sqrt{s_G}(\log \pb \max_jb_j/n)^{1/4}
\]
which is the advertised bound.\end{proof}
\begin{lemma} \label{concated}Let $\A=[\A^1~\dots~\A^s]$ be a concatenation of $s$ matrices. We have that
\[
\|\A\|\leq \sqrt{s}\max_i\|\A^i\|.
\]
\end{lemma}
\begin{proof} Let $\x=[\x^1~\dots~\x^s]$. We have that $\tn{\A\x}\leq \sum_i\tn{\A^i\x^i}\leq \sum_i\|\A^i\|\tn{\x^i}\leq \max_i\|\A^i\|\sqrt{s}$.
\end{proof}

%Next, similar to \irb, we wish to obtain a statistical precision bound. This can again be done by approximating \rrb $\Xbin_R$ in terms of \irb $\Xbin_I$. In particular, we have that
%\[
%\Pc_{\Cc-\Cc}(\Xbin_R^T\z)\leq \Pc_{\Cc-\Cc}(\Xbin_I^T\z)+\sup_{\tn{\vb}\leq 1,\vb\in \Cc-\Cc}\vb^T(\Xbin_R-\Xbin_I)\z
%\]
%Observe that $\vb^T(\Xbin_R-\Xbin_I)\leq $

%\subsection{}
\section{Proof of Theorem \ref{distance}}% and moments of random binned data matrices}
\begin{proof} Assume $\Xbin$ is \rrb~with randomly flipped signs and $\Sb$ is sparse JL matrix described in the theorem. Given unit $\ell_2$ norm $\x$ and an integer $\ell\geq 1$, first, let us find a good representation for $(\tn{\Sb\x}^2-1)^{\ell}$ (identically $(\tn{\Xbin\x}^2-1)^\ell$). Observe that
\begin{align}
\tn{\Sb\x}^2-1=\sum_i (\sum_j\s_{ij}\x_j)^2-1&=\sum_i\sum_{j_1,j_2}\s_{ij_1}s_{ij_2}\x_{j_1}\x_{j_2}-1\\
&=(\sum_i\sum_{j_1\neq j_2}\s_{ij_1}s_{ij_2}\x_{j_1}\x_{j_2})+\sum_i\sum_{j}\s_{ij}^2\x_{j}^2-1\\
&=\sum_i\sum_{j_1\neq j_2}\s_{ij_1}s_{ij_2}\x_{j_1}\x_{j_2}.
\end{align}
This immediately implies that 
\[
\E[(\tn{\Sb\x}^2-1)^\ell]=\E[(\sum_i\sum_{j_1\neq j_2}\s_{ij_1}s_{ij_2}\x_{j_1}\x_{j_2})^{\ell}].
\]
Consequently, we shall focus on analyzing the right hand side by expanding it.
\begin{align}
(\sum_i\sum_{j_1\neq j_2}\s_{ij_1}\s_{ij_2}\x_{j_1}\x_{j_2})^{\ell}=\sum\left( \prod_{i_k,j_{1k}\neq j_{2k},1\leq k\leq \ell}\s_{i_kj_{1k}}\s_{i_kj_{2k}}\x_{j_{1k}}\x_{j_{2k}}\right)\label{sub embed}.
\end{align}
We need to investigate the product on the right hand side \eqref{sub embed}. Each $\prod$ term inside the overall sum is multiplication of multiple $\s_{ij},\x_j$ terms. We first show that, the expectation of the each $\prod$ term is nonnegative given $i_k,j_{1k}\neq j_{2k},1\leq k\leq \ell$. Suppose the expectation is nonzero. In this case, observe that each $\s_{i_kj_{1k}}$ has to be picked even many times. Otherwise $\s_{i_kj_{1k}}$ multiplier will result in zero expectation since it has random sign (scaled $\pm1$ random variable). Given that each $\s_{i_kj_{1k}}$ has even power, $\x_{j_1k}$ has even power as well. Hence, the multiplier of $\prod$ term containing the $\x$ terms is nonnegative. Hence, the nonnegativity boils down to the nonnegativity of 
\begin{align}
\prod_{i_k,j_{1k}\neq j_{2k},1\leq k\leq \ell}\s_{i_kj_{1k}}\s_{i_kj_{2k}}\label{sub embed2}.
\end{align}
Again, since $\s_{i_kj_{1k}}$ has even powers and nonzero, \eqref{sub embed2} is strictly positive. Hence, for any $\prod$ term with nonzero expectation, all terms have even powers and the expectation of the $\prod$ term as well as the right hand side of \eqref{sub embed} is nonnegative.

For the case $\ell=2$, the nonnegative $\s_{ij}$ multipliers in $\prod$ are in the form $\s_{ij}^2\s_{kl}^2$. Below, we show that $\Xbin$ has less number of positive terms in the form \eqref{sub embed2} compared to $\Sb$.
\begin{itemize}
\item When $j$ and $l$ belong to different features (i.e.~$j,l$ corresponds to features $j',l'$ where $j'\neq l'$ so that columns $j$ and $l$ are chosen from different submatrices $\Sb^{j'}$ and $\Sb^{l'}$): $\Xbin$ and $\Sb$ has the same expectation, as the columns of $\Xbin$ corresponding to different features are independent. 
\item When $j=l$ belong to the same feature, $\E[\s_{ij}^2\s_{kj}^2]$ term is same as well, as each column of $\Xbin$ and $\Sb$ have nonzero support chosen uniformly at random of size $s$.
\item Finally, if $j\neq l$ but $j$ and $l$ belong to the same feature, by construction for $\Xbin$, $\E[\s_{ij}^2\s_{kj}^2]=0$ since nonzero supports are non-overlapping. Whereas for $\Sb$, we have $\E[\s_{ij}^2\s_{kj}^2]=\E[\s_{ij}^2]\E[\s_{kj}^2]=n^{-2}$.
\end{itemize}
Hence, $\Xbin$ has better or equal $\ell=2$ moment compared to $\Sb$. Strict inequality occurs if the vector $\x$ contains at least two nonzero values corresponding to same feature ($\nnz{\x^j}>1$ for some $j$) since for that vector, $\Sb$ will have an additional strictly positive moment term.
%In order to conclude with the proof, we will show that positive expectations arising from $\Xbin$ is a subset of that of $\Sb$. To see this, we consider two scenarios.
%\begin{enumerate}
%\item There exists columns $a\neq b\in \{j_{1k},j_{2k}\}_{k\leq \ell}$ such that $(a,b)$ is in the same bin.
%\item The set $\{j_{1k},j_{2k}\}_{k\leq \ell}$ corresponds to columns from distinct bins.
%\end{enumerate}
%In the first case, $\Xbin$ will have no positive expectations since the supports in the same bin corresponding to different columns $a,b$ are nonoverlapping and $\s_{ia}\s_{ib}=0$. In the second case, $\Xbin$ and $\Sb$ have the exact same expectation since both multiplications are dealing with independent columns and nonzero patterns of each column subset picked for $\Xbin$ and $\Sb$ are statistically identical i.e. nonzero support is selected uniformly at random from the sets of size $s$.
\end{proof}

\end{document}